\documentclass{article}

\PassOptionsToPackage{numbers, compress}{natbib}

\usepackage[final]{neurips_2025}

\usepackage[utf8]{inputenc} 
\usepackage[T1]{fontenc}    
\usepackage{url}            
\usepackage{booktabs}       
\usepackage{amsfonts}       
\usepackage{nicefrac}       
\usepackage{microtype}      
\usepackage{xcolor}         
\usepackage{amsmath}
\usepackage[most]{tcolorbox}
\usepackage{enumitem}
\usepackage{graphicx}
\usepackage{amsthm}
\usepackage[table]{xcolor}
\usepackage{multirow}
\usepackage{nicematrix}
\usepackage{wrapfig}
\usepackage{graphicx}
\usepackage{subcaption}

\usepackage{hyperref}

\definecolor{myblue}{HTML}{E1F4FD}
\definecolor{uc}{HTML}{0645AD}

\hypersetup{
  colorlinks=true,
  linkcolor=black,
  urlcolor=uc,
  citecolor=magenta
} 

\theoremstyle{proposition}
\newtheorem{proposition}{Proposition}[section]

\theoremstyle{definition}
\newtheorem{definition}{Definition}[section]

\theoremstyle{theorem}
\newtheorem{theorem}{Theorem}[section]

\theoremstyle{lemma}
\newtheorem{lemma}{Lemma}[section]

\title{HoPE: Hybrid of Position Embedding for \\Long Context Vision-Language Models}

\author{
Haoran Li\textsuperscript{$1$}\quad
Yingjie Qin\textsuperscript{$2$}\quad
Baoyuan Ou\textsuperscript{$2$}\quad
Lai Xu\textsuperscript{$2$}\quad
Ruiwen Xu\textsuperscript{$2$}
\\
${\text{$^1$Carnegie Mellon University}}$ \quad ${\text{$^2$Xiaohongshu Inc.}}$
\\
\texttt{haoranl4@cs.cmu.edu}
}

\begin{document}

\maketitle

\begin{abstract}

Vision-Language Models (VLMs) have made significant progress in multimodal tasks. However, their performance often deteriorates in long-context scenarios, particularly long videos. While Rotary Position Embedding (RoPE) has been widely adopted for length generalization in Large Language Models (LLMs), extending vanilla RoPE to capture the intricate spatial-temporal dependencies in videos remains an unsolved challenge. Existing methods typically allocate different frequencies within RoPE to encode 3D positional information. However, these allocation strategies mainly rely on heuristics, lacking in-depth theoretical analysis. In this paper, we first study how different allocation strategies impact the long-context capabilities of VLMs. Our analysis reveals that current multimodal RoPEs fail to reliably capture semantic similarities over extended contexts. To address this issue, we propose HoPE, a Hybrid of Position Embedding designed to improve the long-context capabilities of VLMs. HoPE introduces a hybrid frequency allocation strategy for reliable semantic modeling over arbitrarily long contexts, and a dynamic temporal scaling mechanism to facilitate robust learning and flexible inference across diverse context lengths. Extensive experiments across four video benchmarks on long video understanding and retrieval tasks demonstrate that HoPE consistently outperforms existing methods, confirming its effectiveness. Our code is available at \url{https://github.com/hrlics/HoPE}.

\end{abstract}

\section{Introduction}
Vision-Language Models (VLMs) \cite{liu2023visual, wang2024qwen2, chen2024internvl, ye2024mplug, lu2024deepseek} have achieved remarkable success in multimodal tasks, including visual question answering \cite{chen2024spatialvlm, xia2025mmedrag, xia2024rule, xia2025mmedagent}, image captioning \cite{yang2023exploring, fei2023transferable}, cross-modal retrieval \cite{li2024generative, wang2024multimodal}, and more~\cite{wang2022learning, wang2024sclip, wang2025scaling}. However, VLMs suffer from significant performance degradation in long-context scenarios, particularly long videos \cite{zhang2024long, chen2024sharegpt4video, cheng2024videollama,maaz2024video}. In such settings, VLMs even struggle with simple tasks like object counting and temporal localization \cite{wang2024needle, wang2025multimodal}, revealing a critical limitation in their ability to model extended spatial-temporal dependencies. This limitation substantially hinders their real-world deployment, where input length often exceeds the context window they have been pretrained on.

Rotary Position Embedding (RoPE) \cite{su2024roformer} has been widely adopted for length generalization in text-based LLMs \cite{touvron2023llama, bai2023qwen, jiang2024mixtral}. Specifically, RoPE incorporates positional information by partitioning the query and key vectors into 2-dimensional pairs and rotating each pair at a unique frequency that decreases as the dimensional index increases. Despite its advantages, directly applying 1D RoPE fails to capture the intricate spatial-temporal dependencies in videos. Several methods have been proposed to extend 1D RoPE for multimodal inputs \cite{wang2024qwen2, heo2024rotary, wei2025videorope}. Among these, the most common approach is to allocate different frequencies to encode different positional components, as shown in the upper plots of Figure~\ref{fig:1}. For example, M-RoPE \cite{wang2024qwen2} allocates the \textit{highest} frequencies for temporal modeling (\(t\)), and the remaining low frequencies for spatial modeling (\(x,y\)). In contrast, VideoRoPE \cite{wei2025videorope} proposes to allocate the \textit{lowest} frequencies for temporal dimensions (\(t\)), and further applies a fixed scaling factor to scale the temporal indices of visual tokens, as shown in the lower plots of Figure~\ref{fig:1}. Despite their improved performance, two significant challenges remain unsolved. Firstly, current methods mainly rely on heuristics rather than theoretical analysis to determine the frequency allocation strategy. Second, applying a fixed and unidirectional scaling factor for all videos is suboptimal in real-world scenarios, where videos proceed at different speeds and demonstrate significant variance in information densities.

\begin{figure}[t]
    \centering
    \includegraphics[width=0.98\textwidth]{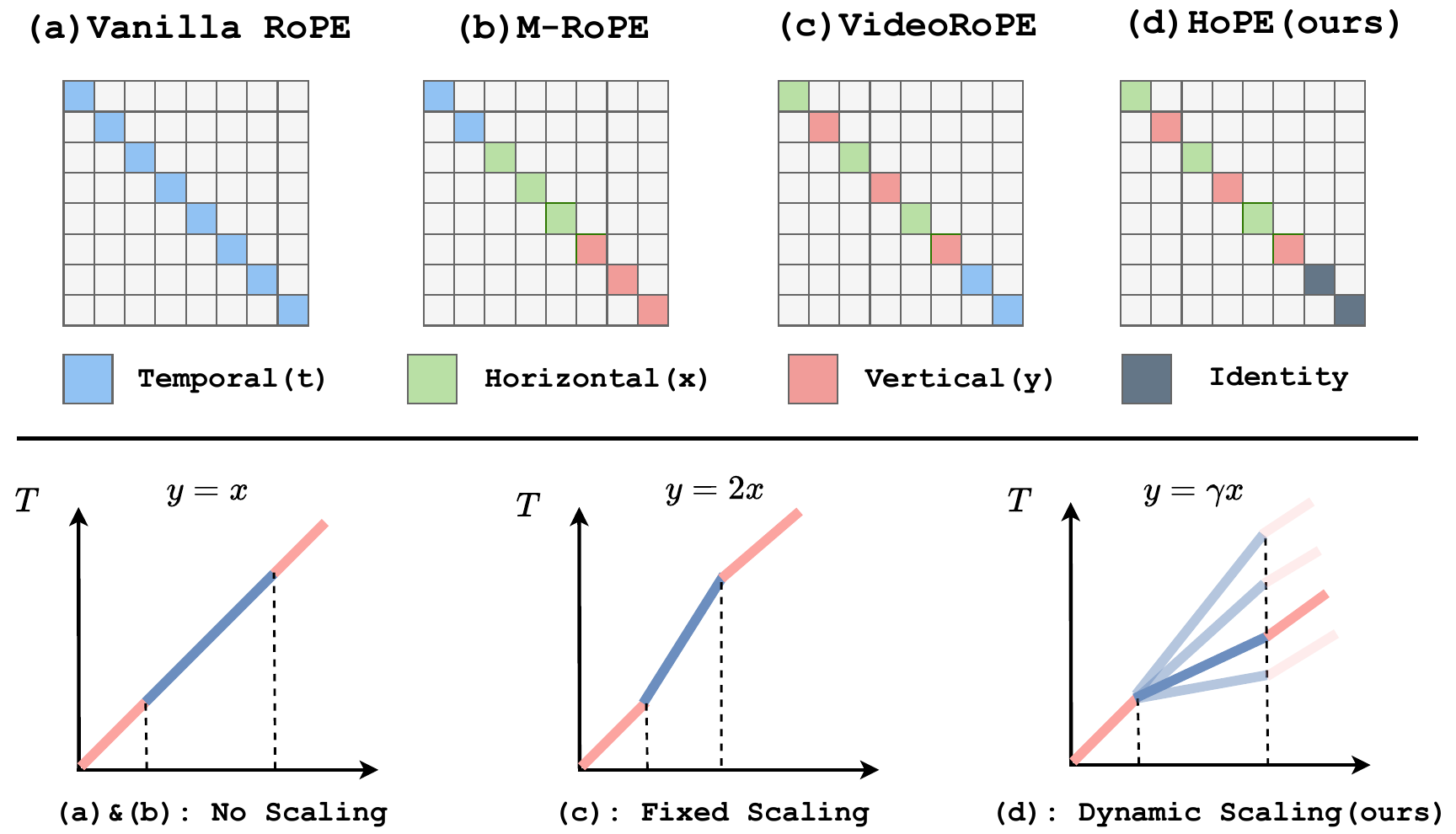}
    \caption{\textbf{Comparison of our HoPE and existing methods.} Upper plots illustrate the frequency allocation strategies in different RoPE variants. Here, frequency decreases along the diagonal. (d) HoPE sets the lowest frequencies to zero for reliable long-range semantic modeling. Lower plots demonstrate different temporal scaling mechanisms. (d) HoPE proposes dynamic and bidirectional scaling to learn temporal dynamics at multiple scales, facilitating robustness to various video speeds.}
     \vspace{-1.5em}
    \label{fig:1}
\end{figure}

To address these challenges, we begin with an in-depth theoretical analysis in Section~\ref{sec:analysis}, outlining the ideal properties that a multimodal RoPE should possess. Our analysis reveals that: (1) the flattening operation in vanilla RoPE inherently violates spatial-temporal locality prior, which is crucial in video modeling; (2) despite diverse frequency allocation strategies, existing multimodal RoPE variants fail to reliably capture semantic similarities over extended contexts; (3) temporal scaling of visual tokens should include both compression and expansion to accommodate varying video speeds in real-world scenarios. Guided by these insights, we propose HoPE, a \textbf{H}ybrid \textbf{o}f \textbf{P}osition \textbf{E}mbedding designed to improve the long-context capabilities of VLMs. As shown in Figure~\ref{fig:1}, HoPE first introduces a hybrid
frequency allocation strategy to facilitate long-range semantic modeling. In this strategy, higher frequencies, which are more sensitive to positional differences and better at capturing local features, are allocated to spatial components (\(x,y\)) in an interleaved manner. Meanwhile, the lower frequencies are directly set to zero and allocated to temporal component (\(t\)) to enable reliable semantic modeling. Moreover, HoPE develops a dynamic temporal scaling mechanism for lengthy visual tokens. This mechanism not only enhances VLMs' robustness to various video speeds, which are common in real-world scenarios, but also offers flexible scaling during inference across diverse context lengths.

We summarize our contributions as follows:
\vspace{-2pt}
\begin{itemize}[leftmargin=8pt, itemsep=-2pt]
    \item To our best knowledge, we provide the first theoretical analysis of how different frequency allocation strategies in multimodal RoPEs impact the long-context capabilities of VLMs, offering insights for the design and analysis of future multimodal RoPEs.
    
    \item Guided by our analysis, we propose HoPE, which consists of a hybrid frequency allocation strategy for reliable semantic modeling in long-context scenarios, and a dynamic temporal scaling mechanism for robust and flexible temporal comprehension.
        
    \item Extensive experiments on four video benchmarks demonstrate that HoPE consistently outperforms existing RoPE variants, achieving improvements of 22.23\% and 8.35\% on long video retrieval and long video understanding tasks, confirming its effectiveness.
\end{itemize}

\section{Preliminaries}
\label{sec:preli}
\textbf{Rotary Position Embedding (RoPE).} Current Transformer-based LLMs rely on Positional Encodings (PEs) to incorporate sequential information into the attention mechanism. Among various PEs, Rotary Position Embedding (RoPE) \cite{su2024roformer} has emerged as a dominant approach for long-context modeling in text-based LLMs \cite{touvron2023llama, bai2023qwen, jiang2024mixtral}. The key to RoPE's success lies in its ability to encode \textit{relative} position information through an \textit{absolute} positional encoding approach, ensuring both effectiveness and efficiency. Consider query and key vectors with \(d\) dimensions (where \(d\) is an even number), RoPE partitions the dimensions into \(d/2\) pairs, e.g., \(\mathbf{q}_n=[\mathbf{q}_n^{(0)}; \mathbf{q}_n^{(1)};\dots;\mathbf{q}_n^{(d/2-1)}]\). Each pair of dimensions is assigned a unique rotation angle \(\theta_i=b^{-2i/d}, i \in \{0,1,\dots,d/2-1\}\), where \(b\) is a pre-defined frequency base and set to \(10,000\) by default \cite{su2024roformer}. This rotation can be achieved through a rotation matrix as follows:

\begin{equation}
    r^{(i)}=
    \begin{pmatrix}
        \cos\theta_i & -\sin\theta_i \\
        \sin\theta_i  &  \cos\theta_i \\
    \end{pmatrix}.
\end{equation}

The overall rotation matrix \(\mathbf{R}_n\) is constructed by concatenating each pair's rotation matrix along the diagonal to form a block-diagonal matrix, i.e., \(\mathbf{R}_n=\operatorname{diag}(r^{(0)},r^{(1)},\dots,r^{(d/2-1)}) \in \mathbb{R}^{d \times d}\). Therefore, during attention computation, the attention score\footnote{Here, we omit the softmax function and \(1/\sqrt{d}\) scaling in standard Transformer \citep{vaswani2017attention} for simplicity.} \(\mathbf{A}_{n,m}\) between the \(n\)-th query \(\mathbf{q}_n\) and \(m\)-th key \(\mathbf{k}_m\) is:

\begin{equation}
    \mathbf{A}_{n,m}=(\mathbf{q}_n\mathbf{R}_n)(\mathbf{k}_m\mathbf{R}_m)^{\top}=\mathbf{q}_n\mathbf{R}_{n-m}\mathbf{k}_m^{\top},
\end{equation}

where the rotation matrix \(\mathbf{R}_{n-m}\) can be formulated as:

\begin{equation}
    \small
    \mathbf{R}_{n-m}=
    \begin{pmatrix}
        \cos\theta_0(n-m) & -\sin\theta_0(n-m)  & \cdots & 0 & 0 \\
        \sin\theta_0(n-m)  &  \cos\theta_0(n-m)  & \cdots & 0 & 0 \\
        \vdots & \vdots & \ddots & \vdots & \vdots \\
        0 & 0 & \cdots & \cos\theta_{(d/2-1)}(n-m)  & -\sin\theta_{(d/2-1)}(n-m) \\
        0 & 0 & \cdots & \sin\theta_{(d/2-1)}(n-m) &  \cos\theta_{(d/2-1)}(n-m)
    \end{pmatrix}.
    \label{eq:1D_RoPE}
\end{equation}

It can be observed that through pairwise attention computation, the final rotation matrix naturally incorporates the \textit{relative} position information \((n-m)\) between the query-key pair.

\textbf{No Positional Encoding (NoPE).} Despite the popularity of RoPE, several works have pointed out that the \textit{causal} attention mechanism in current decoder-only LLMs implicitly learns \textit{absolute} positional information \cite{haviv2022transformer, kazemnejad2023impact, wang2024length}. This motivates the development of No Positional Encoding (NoPE). Specifically, the \textit{causal} attention mask enforces $A_{n,m} = 0$ for all $m > n$, ensuring that each token only attends to itself and previous tokens. Under this constraint, the attention score with NoPE is a simple dot product between the query vector \(\mathbf{q}_n\) and the key vector \(\mathbf{k}_m\), i.e., \(\mathbf{A}_{n,m}=\mathbf{q}_n\mathbf{k}_m^{\top}\), providing no explicit positional information to Transformers.





\section{Analysis}
\label{sec:analysis}
In this section, we conduct a comprehensive theoretical analysis of multimodal RoPE variants, aiming to answer the following questions: \textbf{(1)} Is vanilla RoPE enough for long-context VLMs? \textbf{(2)} How do different frequency allocation strategies impact semantic modeling in long-range multimodal contexts? \textbf{(3)} How should we assign the temporal indices for text and visual tokens?

\subsection{Vanilla RoPE Fails in Spatial-Temporal Structure}

Several recent VLMs \cite{liu2023visual, lu2024deepseek, maaz2024video, zhang2023video, lin2024video, li2024lite} still use vanilla RoPE for multimodal inputs. In their approach, each video frame is first encoded by a vision encoder (e.g., ViT \cite{dosovitskiy2021an}) and then flattened into a sequence of patch-level tokens. These visual tokens will be treated equally as text tokens for positional encoding, with each token incorporating only 1D temporal information. We show in Proposition~\ref{prop:1drope_for_video} that this approach, while easy to implement, distorts spatial-temporal localities and fundamentally limits VLMs' ability to model extended spatial-temporal dependencies.

\begin{proposition}[1D RoPE violates spatial-temporal locality priors] 
Given any query \(\mathbf{q}\) at position \((t,x,y)\) and a relative distance of 1 in spatial or temporal dimensions, the flattening operation in 1D RoPE distorts the relative distance with a magnitude dependent on the frame resolution.
\label{prop:1drope_for_video}
\end{proposition}

We provide the proof in Appendix~\ref{app:proof_1drope}. This mismatch between positional encoding and the 3D structure of videos creates distorted attention patterns, making it difficult for models to learn meaningful spatial-temporal relationships essential for video-related tasks.


\begin{tcolorbox}[
  colback=myblue,  
  colframe=black,  
  arc=0.8mm,                 
  boxrule=1pt,         
  left=2mm, right=2mm, top=1mm, bottom=1mm,
  enhanced,
  width=\linewidth,
]
\textbf{Conclusion 1.} Directly applying vanilla RoPE to multimodal long-context inputs inherently fails to capture their complex spatial-temporal dependencies.
\end{tcolorbox}

\subsection{Current Multimodal RoPEs Are Unreliable in Long-Range Semantic Modeling}

To capture the spatial-temporal structure of multimodal inputs, a recent VLM, Qwen2-VL \cite{wang2024qwen2}, has introduced a Multimodal Rotary Position Embedding (M-RoPE). Concretely, M-RoPE partitions the $128$-dimensional rotary embedding into three distinct groups: the first $32$ dimensions for temporal information ($t$), the subsequent $48$ dimensions for horizontal spatial information ($x$), and the final $48$ dimensions for vertical spatial information ($y$), i.e., \(\mathbf{R}_{t,x,y} = \operatorname{diag}(\mathbf{R}_t, \mathbf{R}_x,\mathbf{R}_y)\). While this approach realizes a naive extension for RoPE, a fundamental question remains to be answered: 
\begin{center}
    \textit{How do different frequency allocation strategies impact the performance of multimodal RoPE?} 
\end{center}
This question arises from the fact that in RoPE, different dimensions carry unique frequencies ($\theta_{i}=b^{-2i/d}$, $i \in \{0,1,\dots,d/2-1\}$), as shown in Equation~\ref{eq:1D_RoPE}. Therefore, different strategies exist for frequency allocation in multimodal RoPE. As shown in Figure~\ref{fig:1}, M-RoPE allocates the highest frequencies for \(t\), intermediate frequencies for \(x\), and the lowest frequencies for \(y\). In contrast, VideoRoPE \cite{wei2025videorope} proposes to assign the lowest frequencies to temporal modeling \((t)\) and high frequencies to spatial dimensions \((x,y)\). Their empirical justification stems from attention pattern analysis, which reveals that dimensions encoded with the lowest frequencies exhibit a more pronounced \textit{attention sink} phenomenon \cite{xiao2024efficient}, which has proven to be effective in long-context modeling. However, we argue that using the lowest frequencies for temporal modeling is still unreliable in capturing semantic similarities in extended multimodal contexts. Specifically, we first introduce semantic preference, a property where attention mechanisms should prioritize semantically similar tokens regardless of their relative distance, and formally define this concept in Definition~\ref{def:semantic_pref}.

\begin{definition}[Semantic Preference]
For any query vector $\mathbf{q}$ and a semantically similar key vector $\mathbf{k}'$ that can be expressed as $\mathbf{k}' = \mathbf{q} + \delta$ where $\delta$ is a zero-mean perturbation, the attention score with RoPE should satisfy:
\begin{equation*}
\mathbb{E}_{\mathbf{q},\mathbf{k},\delta}[\mathbf{q}\mathbf{R}_{\Delta t \Delta x \Delta y }\mathbf{k}'^{\top}-\mathbf{q}\mathbf{R}_{\Delta t \Delta x \Delta y}\mathbf{k}^{\top}]\ge 0,
\end{equation*}
where $\mathbf{k}$ is the key vector of a semantically unrelated token. This preference should hold regardless of the relative distance (\(\Delta t, \Delta x, \Delta y\)) between the query and key.
\label{def:semantic_pref}
\end{definition}

Then, we show in Theorem~\ref{theorem:semantic_preference} that \textbf{\textit{all}} frequency allocation strategies of current multimodal RoPEs, including selecting the highest/lowest frequencies for temporal modeling, are unreliable in maintaining the semantic preference property over extended contexts. This limitation arises because, with sufficiently long contexts, even the lowest frequencies can produce arbitrary rotations, ultimately undermining semantic preference. We provide the proof in Appendix~\ref{app:semantic_pref}.

\begin{theorem}
    Let $X = [x_1, x_2, \dots, x_L]$ be an input sequence, and let RoPE use any fixed set of temporal frequencies (e.g., highest or lowest). Then there exists a critical length $L_c$ such that for all $L \ge L_c$, the semantic preference property (Definition~\ref{def:semantic_pref}) is violated.
    \label{theorem:semantic_preference}
\end{theorem}

\begin{tcolorbox}[
  colback=myblue,   
  colframe=black,   
  arc=1mm,                 
  boxrule=1pt,             
  left=2mm, right=2mm, top=1mm, bottom=0.8mm,
  enhanced,
  width=\linewidth,
]
\textbf{Conclusion 2.} There exist different frequency allocation strategies to extend vanilla
RoPE to multimodal RoPE. However, we prove that none of these strategies can reliably
maintain the semantic preference property over a sufficiently long context.
\end{tcolorbox}

\subsection{How to Assign Positional Index for Multimodal Inputs?}

Currently, most VLMs \cite{liu2023visual, wang2024qwen2, chen2024internvl, ye2024mplug, cheng2024videollama, maaz2024video} adopt the same temporal stride for video frames and text tokens, as shown in Figure~\ref{fig:1}. However, this approach overlooks the inherent difference in information densities between text and visual tokens. To address this issue, VideoRoPE \cite{wei2025videorope} applies a fixed scaling factor (implemented as \(2\)) to adjust the temporal indices of visual tokens, achieving better empirical performance. However, this rigid scaling approach lacks the flexibility needed for diverse real-world videos, which naturally vary in pace and information density. A more ideal approach would incorporate both temporal compression and expansion capabilities, allowing the model to learn multi-scale temporal relationships, thereby enabling more robust temporal modeling.

\begin{tcolorbox}[
  colback=myblue,  
  colframe=black,   
  arc=1mm,        
  boxrule=1pt,              
  left=2mm, right=2mm, top=1mm, bottom=0.8mm,
  enhanced,
  width=\linewidth,
]
\textbf{Conclusion 3.} Temporal index scaling of visual tokens is crucial for balancing multimodal information, yet current methods lack flexibility and bidirectionality.
\end{tcolorbox}

\section{HoPE: Hybrid of Position Embedding for Long Context VLMs}
\label{sec:method}

To address the above challenges, we propose HoPE, a \textbf{H}ybrid \textbf{o}f \textbf{P}osition \textbf{E}mbedding designed to improve the long-context capability of VLMs. As illustrated in Figure~\ref{fig:1} and Figure~\ref{fig:all_freq}, HoPE first introduces a hybrid frequency allocation (HFA) strategy to better preserve the semantic preference property (Definition~\ref{def:semantic_pref}) in long-context modeling. Under this strategy, spatial information will be encoded with higher frequencies to capture local semantics, while the lowest frequencies will be set to zero (as in NoPE \cite{haviv2022transformer}) to facilitate long-range semantic modeling. Second, HoPE develops a dynamic temporal scaling (DTS) mechanism to enhance VLMs' robustness to various video speeds and enable flexible inference under diverse context lengths. We will detail these strategies as follows:

\subsection{Hybrid Frequency Allocation Strategy}

To extend vanilla RoPE to multimodal scenarios, a common approach is to allocate different frequencies to encode different positional components \((t,x,y)\). For example, M-RoPE \cite{wang2024qwen2} assigns the highest frequencies for temporal modeling and lower frequencies for spatial encoding. In contrast, VideoRoPE \cite{wei2025videorope} allocates the lowest frequencies for temporal modeling, achieving better empirical results. However, in Theorem~\ref{theorem:semantic_preference}, we theoretically prove that, despite using lower frequencies being more ideal for semantic modeling, none of these frequency allocation strategies can maintain the ideal semantic preference property (Definition~\ref{def:semantic_pref}) over extended contexts.

To provide a stronger theoretical guarantee for the semantic preference property, we propose a hybrid frequency allocation strategy. As shown in Figure~\ref{fig:1}, we encode spatial information \((x,y)\) with high frequencies, as high frequencies are more sensitive to positional differences and thereby better at capturing local semantics \cite{wei2025videorope, barbero2025round}. Following existing work \cite{wei2025videorope}, \(x\) and \(y\) are encoded in an interleaved manner to prevent biased spatial encoding. More importantly, unlike existing methods \cite{su2024roformer, wei2025videorope, wang2024qwen2}, we directly set the lowest frequencies to zero (as in NoPE \cite{haviv2022transformer}) to provide a stronger guarantee for the semantic preference property (Definition~\ref{def:semantic_pref}), as shown in Figure~\ref{fig:all_freq}. Specifically, for \(d=128\), we interleave \(x\) and \(y\) positions in the first \(96\) dimensions of the rotation matrix and set the frequencies in the remaining \(32\) dimensions to zero, which corresponds to an identity matrix:

\[
    \tiny
    R_{\Delta x, \Delta y} =  \operatorname{diag}(
    \setlength{\arraycolsep}{0.5pt} 
 \begin{pmatrix}
    \cos\theta_0\Delta x & -\sin\theta_0\Delta x  & 0 & 0 & \cdots & 0 & 0 & 0 & 0\\
    \sin\theta_0\Delta x  &  \cos\theta_0\Delta x  & 0 & 0 & \cdots & 0 & 0 & 0 & 0\\
    0 & 0 & \cos\theta_1\Delta y & -\sin\theta_1\Delta y & \cdots & 0 & 0 & 0 & 0 \\
    
    0 & 0 & \sin\theta_1\Delta y & \cos\theta_1\Delta y & \cdots & 0 & 0 & 0 & 0\\
    
    \vdots & \vdots & \vdots & \vdots & \ddots & \vdots & \vdots & \vdots & \vdots\\
    
    0 & 0 & 0 & 0 & \dots & \cos\theta_{46}\Delta x & -\sin\theta_{46}\Delta x & 0 & 0 \\
    
    0 & 0 & 0 & 0 & \dots & \sin\theta_{46}\Delta x & \cos\theta_{46}\Delta x & 0 & 0 \\
    
     0 & 0 & 0 & 0 & \dots & 0 & 0&  \cos\theta_{47}\Delta y & -\sin\theta_{47}\Delta y\\
     
    0 & 0 & 0 & 0 & \dots &  0 & 0 &  \sin\theta_{47}\Delta y & \cos\theta_{47}\Delta y\\

     \end{pmatrix}
     ,
     I_{32})
\]

We now provide a theoretical analysis of how this hybrid strategy helps the attention mechanism to capture long-range semantic similarities. Building on Definition~\ref{def:semantic_pref} and Theorem~\ref{theorem:semantic_preference}, we first formalize the condition under which semantic preference is preserved in multimodal RoPE.

\begin{figure}[t]
    \centering
    \begin{subfigure}[b]{0.32\textwidth}
        \centering
        \includegraphics[width=\textwidth]{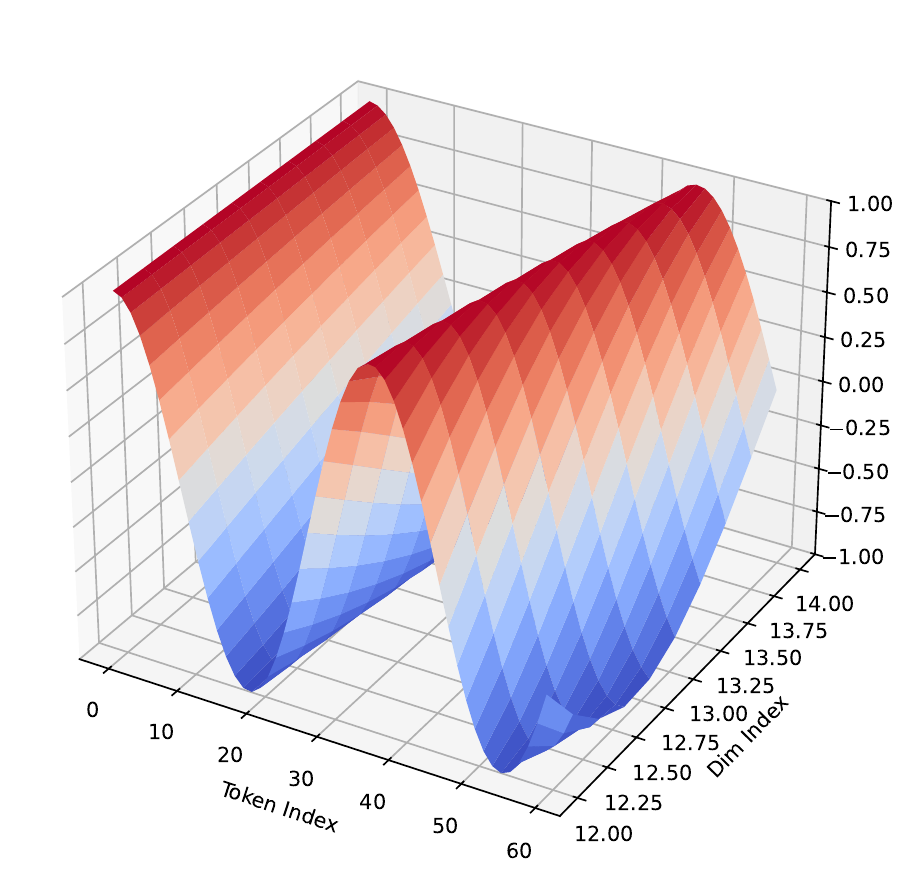}
        \caption{High frequencies for temporal modeling in M-RoPE.}
        \label{fig:m_rope}
    \end{subfigure}
    \hfill
    \begin{subfigure}[b]{0.32\textwidth}
        \centering
        \includegraphics[width=\textwidth]{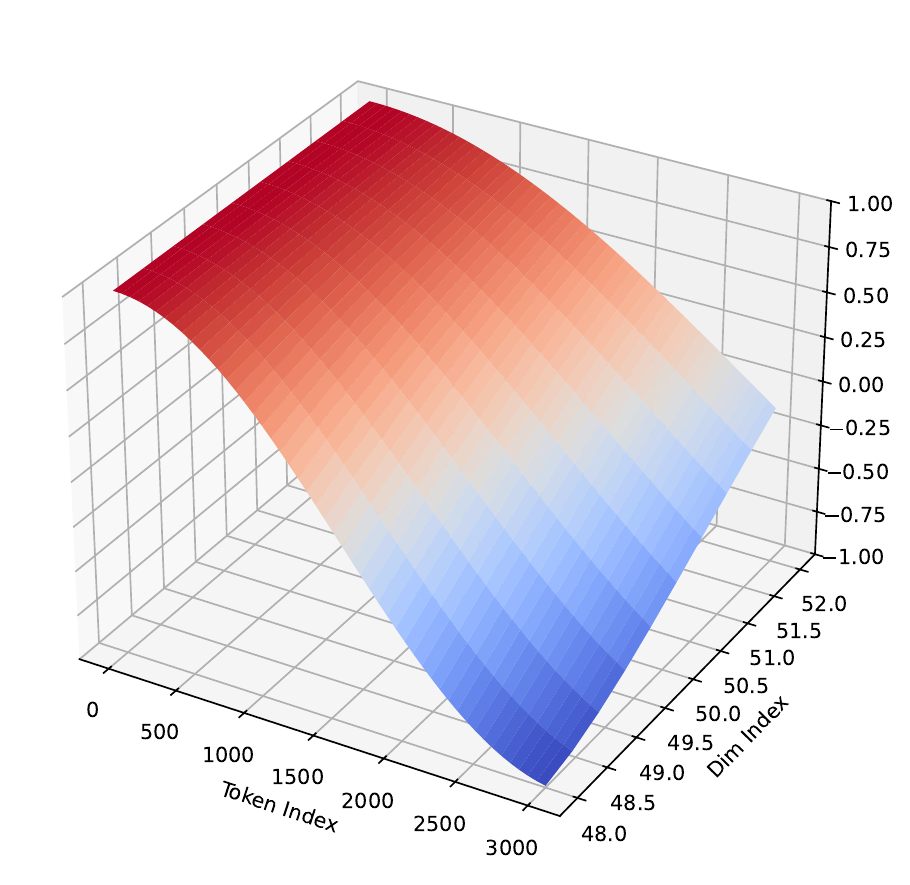}
        \caption{Low frequencies for temporal modeling in VideoRoPE.}
        \label{fig:video_rope}
    \end{subfigure}
    \hfill
    \begin{subfigure}[b]{0.32\textwidth}
        \centering
        \includegraphics[width=\textwidth]{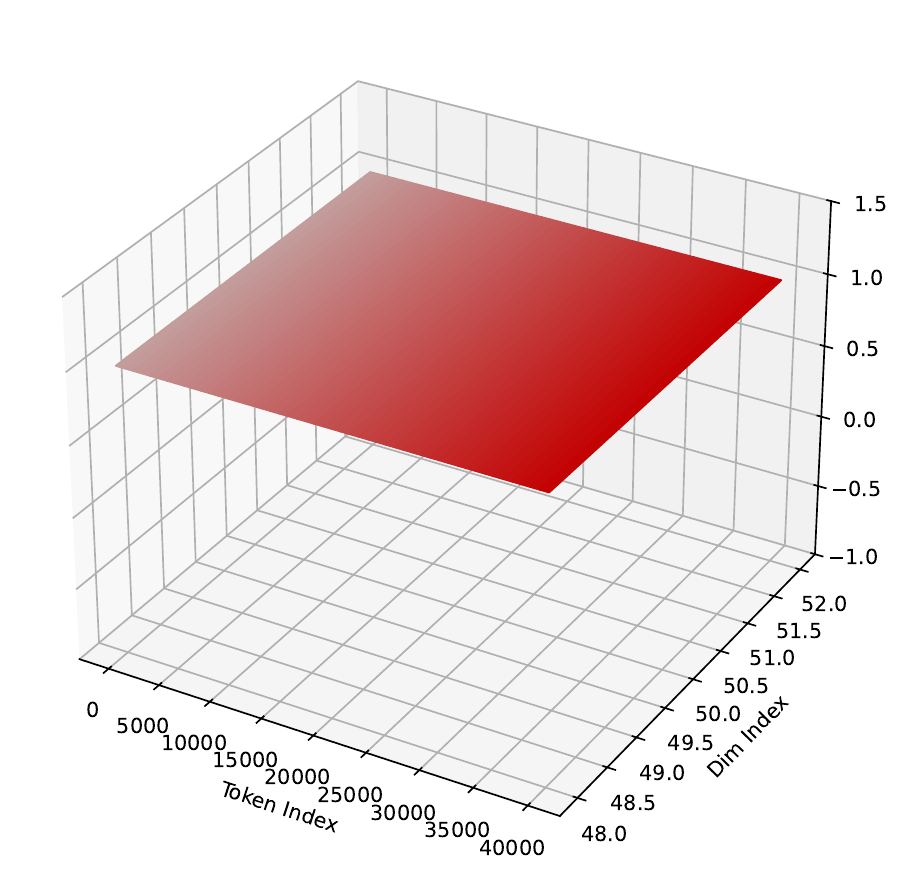}
        \caption{Zero frequencies for temporal modeling in HoPE (ours).}
        \label{fig:hope}
    \end{subfigure}
    \caption{\textbf{Multimodal RoPEs use different frequencies for temporal modeling.} M-RoPE uses the \textit{highest }frequencies, which are suboptimal for long-context modeling. VideoRoPE utilizes the \textit{lowest} frequencies for more stable semantic modeling. Our HoPE, employing \textit{zero} frequencies for temporal modeling, establishes the upper bound of semantic modeling capabilities across all strategies.}
    \label{fig:all_freq}
\vspace{-1em}
\end{figure}

In particular, Lemma~\ref{lemma:semantic_pref_condition} establishes a clear theoretical criterion for maintaining semantic preference with multimodal RoPE. It directly follows from our analysis in Theorem~\ref{theorem:semantic_preference} and Appendix~\ref{app:semantic_pref}, providing the theoretical foundation for our proposed method.

\begin{tcolorbox}[
  colback=blue!5!white,  
  colframe=black,  
  arc=0.8mm,                 
  boxrule=1pt,         
  left=2mm, right=2mm, top=1mm, bottom=1mm,
  enhanced,
  width=\linewidth,
]
\begin{lemma}[Necessary Condition for Semantic Preference]
For a multimodal RoPE with rotation matrix $\mathbf{R}_{t,x,y} = \operatorname{diag}(\mathbf{R}_t, \mathbf{R}_x,\mathbf{R}_y)$, the semantic preference property (Definition~\ref{def:semantic_pref}) holds if, for all possible relative distances,
\begin{equation*}
\sum_{i \in i_t}2\sigma^2\mathrm{cos}(\Delta t \cdot \theta_i)+\sum_{i \in i_x}2\sigma^2\mathrm{cos}(\Delta x \cdot \theta_i)+\sum_{i \in i_y}2\sigma^2\mathrm{cos}(\Delta y \cdot \theta_i) \geq 0,
\label{eq:sem_pref}
\end{equation*}
where \(\sigma^2\) is the variance of each component in the query/key vector, \(i_t, i_x,i_y\) are dimensions allocated to \(t,x,y\), and $\Delta t\in\{0,1,\dots,L-1\}$, $\Delta x\in\{0,1,\dots,H\}$, $\Delta y\in\{0,1,\dots,W\}$.
\label{lemma:semantic_pref_condition}
\end{lemma}
\end{tcolorbox}

Based on this Lemma, we now prove how our hybrid frequency allocation strategy provides stronger guarantees for the semantic preference property over extended contexts. Specifically, HFA set \(\theta_i = 0\) for all \(i \in i_t\). Hence, the temporal terms in Lemma~\ref{lemma:semantic_pref_condition} reduce to \(\sum_{i \in i_t} 2\sigma^2 \cdot 1\), noting that \(\sum_{i \in i_t} 2\sigma^2 \cdot 1 \geq \sum_{i \in i_t} 2\sigma^2 \cos(\Delta t \cdot \theta_i)\) holds for any choice of temporal frequencies \(\theta_i\). Adding the identical spatial terms on both sides, we obtain:

\begin{equation}
\begin{aligned}
&\sum_{i \in i_t} 2\sigma^2 \cdot 1 + \sum_{i \in i_x} 2\sigma^2 \cos(\Delta x \cdot \theta_i) + \sum_{i \in i_y} 2\sigma^2 \cos(\Delta y \cdot \theta_i) \\
&\geq \sum_{i \in i_t} 2\sigma^2 \cos(\Delta t \cdot \theta_i) + \sum_{i \in i_x} 2\sigma^2 \cos(\Delta x \cdot \theta_i) + \sum_{i \in i_y} 2\sigma^2 \cos(\Delta y \cdot \theta_i).
\end{aligned}
\end{equation}

This shows that our HFA strategy, by setting the lowest frequencies to zero, dominates any other choice of temporal frequencies and provides a stronger guarantee for preserving the semantic preference property under long-context scenarios, as in Theorem~\ref{theorem:hybrid_frequency}.

\begin{theorem}
For multimodal position embeddings with dimensions allocated across temporal (\(i_t\)), and spatial components (\(i_x,i_y\)), setting \(\theta_i = 0\) for all temporal dimensions \(i \in i_t\) maximizes the semantic preference guarantee in Definition~\ref{def:semantic_pref}, compared to any alternative frequency allocation strategy, particularly under extended context lengths.
\label{theorem:hybrid_frequency}
\end{theorem}

Another interesting finding is that, if we set \(|i_t|=d/4,|i_x|=|i_y|=d/8\) and \(\theta_i=0, i \in i_t\), for any context length \(t\) and spatial size \(x,y\), semantic preference property invariably holds, as Lemma~\ref{lemma:semantic_pref_condition} reduces to \(\sum_{i=0}^{d/8-1}2\sigma^2(2+\mathrm{cos}(\Delta x \cdot \theta_{2i})+\mathrm{cos}(\Delta y \cdot \theta_{2i+1})) \geq 0.\) However, the empirical results of this approach are inferior to our proposed HoPE, probably due to the decreased number of frequencies for spatial modeling. More discussions are provided in Appendix~\ref{app:ideal_sem_pref}.

\subsection{Dynamic Temporal Scaling Mechanism}
Considering the distinct information densities of text and visual tokens, HoPE introduces a dynamic temporal scaling mechanism that adjusts the temporal strides of visual inputs. Specifically, we first define a set of scaling factors, e.g., \(\Gamma =\{0.5, 0.75, 1, 1.25,1.5\}\), which includes both stretching (\(\gamma>1\)) and compressing \((\gamma<1\)) operations. During training, the scaling factor \(\gamma\) is randomly selected from \(\Gamma\) and applied to each video. This allows the model to learn temporal relationships at multiple scales, making it more robust to variations in video speed, which are common in real-world scenarios. Consider a multimodal input \((text, video, text)\) of length \(L_t\), \(L_v\), and \(L_e\), respectively. The position indices \((t,x,y)\) for each token with our dynamic scaling factor \({\color{blue}\gamma}\) are:

\begin{equation}
    \small
    (t, x, y) = 
    \begin{cases} 
    (l, l, l), &  0 \leq l < L_t \\\\
    \begin{pmatrix}
    L_t + {\color{blue}\gamma} (l - L_t), \\
    L_t + {\color{blue}\gamma} (l - L_t) + w - \frac{W}{2}, \\
    L_t + {\color{blue}\gamma} (l - L_t) + h - \frac{H}{2}
    \end{pmatrix}, &  L_t \leq l < L_t + L_v \\\\
    \begin{pmatrix}
    ({\color{blue}\gamma} - 1) L_v + l, \\
    ({\color{blue}\gamma} - 1) L_v + l, \\
    ({\color{blue}\gamma} - 1) L_v + l
    \end{pmatrix}, & L_t + L_v \leq l < L_t + L_v + L_e
    \end{cases}.
    \label{eq:dts}
\end{equation}

Note that for visual tokens ($L_t \leq l < L_t + L_v$), $l - L_t$ indicates the distance of the current frame from the start frame. During inference, scaling factors can be flexibly selected from the set to accommodate videos of different lengths. It is worth noting that unlike existing methods, which do not consider temporal scaling for visual tokens \cite{liu2023visual, wang2024qwen2, ye2024mplug, lu2024deepseek} or just apply a fixed and unidirectional scaling factor for both training and testing \cite{wei2025videorope}, our methods not only help the model learn temporal relationships at multiple scales, but also offer flexibility during inference to accommodate various context lengths.

\section{Experiment}
In this section, we evaluate the performance of HoPE on four video benchmarks across long video understanding and long video retrieval tasks, aiming to validate its effectiveness in multimodal long-context modeling. Additionally, we conduct ablation studies to investigate the individual contribution of each strategy to overall performance and the interplay between task type, context length, and scaling factor selection.

\subsection{Experimental Setups}
\label{sec:exp_setup}

\textbf{Implementation Details.} We utilize Qwen2-1.5B and Qwen2-7B \cite{yang2024qwen2} as the backbone models. By integrating these models with vision encoders from Qwen2-VL-2B/7B-Instruct \cite{wang2024qwen2}, we obtain Qwen2-2B/7B-Video, respectively. During training, we adopt a batch size of 128, a learning rate of 1e-5(2B)/2e-5(7B) with a cosine scheduler. Following the instruction tuning settings in Qwen2-VL \cite{wang2024qwen2}, we set the maximum video frames to 128 and the video sampling rate to 2. The training context length is set to 8k, with the entire training process taking approximately 304 GPU hours on machines equipped with H800-80GB GPUs. During evaluation, the minimum tokens per frame are set to 144.

\textbf{Training Data.} We train the models on a subset of LLaVA-Video-178k \cite{zhang2024video}, which consists of 178k videos ranging from 0 to 3 minutes and 5M instruction samples, including captions, free-form, and multiple-choice question answering. Our selected subset includes 30k videos with durations under 2 minutes and 3k videos with durations between 2 and 3 minutes, resulting in roughly 300k pairs. 

\textbf{Baselines.} We compare HoPE with the following RoPE variants: 1) vanilla RoPE \cite{su2024roformer}, the standard approach in long-context LLMs, 2) M-RoPE \cite{wang2024qwen2}, a famous RoPE extension in Qwen2-VL for multimodal inputs, 3) VideoRoPE \cite{wei2025videorope}, a specialized RoPE variant designed for video-related tasks.

\textbf{Evaluation Benchmarks.} We evaluate HoPE across four video benchmarks for long video understanding and long video retrieval tasks. For long video understanding, we utilize LongVideoBench \cite{wu2024longvideobench}, Video-MME \cite{fu2024video}, and MLVU \cite{zhou2024mlvu}, covering videos ranging from a few seconds to 2 hours. For long video retrieval, we employ V-NIAH (Visual Needle-In-A-Haystack) \cite{zhang2024long}. In this task, a "needle" image is randomly inserted into a "haystack" video, and the
VLM is required to answer a question specifically about the embedded "needle" image. Following the protocol in V-NIAH \cite{zhang2024long}, we utilize a haystack video with 1-hour duration (3,000 frames) and insert the needle image at 20\% depth intervals (e.g., a frame depth of 0\% would place the needle image at the very beginning of the video). For more detailed benchmark descriptions, please refer to Appendix~\ref{app:benchmark}.

\begin{table}[t]
\caption{\textbf{Performance comparison on long video understanding benchmarks.} The training context length of all methods is set to 8k, and we report the performance on 8k, 16k, 32k, and 64k to evaluate length generalization. The best results are \textbf{bold}, while the second best results are \underline{underlined}.}
\vspace{0.8em}
\centering
\renewcommand{\arraystretch}{1.25}
\resizebox{\textwidth}{!}{
\begin{tabular}{l|cccc|cccc|cccc}
\toprule[1pt]
 & \multicolumn{4}{c|}{\textbf{MLVU}} & \multicolumn{4}{c|}{\textbf{LongVideoBench}} & \multicolumn{4}{c}{\textbf{Video-MME}} \\
\cmidrule(lr){2-5} \cmidrule(lr){6-9} \cmidrule(lr){10-13}
Method & 8k & 16k & 32k & 64k & 8k & 16k & 32k & 64k & 8k & 16k & 32k & 64k \\
\midrule
\rowcolor{gray!15}
\multicolumn{13}{c}{\rule[-5pt]{0pt}{16pt}\textit{ \hspace{2em} \textbf{Qwen2-2B-Video}}} \\
\midrule
Vanilla RoPE  & {\bf 55.10} & \underline{55.21} & 54.36 & 39.06 & 51.57 & 50.29 & 51.00 & 34.21 & 50.70 & \underline{51.48} & 51.44 & 20.31 \\
M-RoPE  
& 53.26 & 53.69 & \underline{54.73} & 40.63
& 50.81 & \underline{52.26} & 51.30 & \underline{44.74} & \underline{51.44} & 51.22 & \underline{51.52} & \underline{23.44} \\
VideoRoPE & 54.75 & 55.19 & 54.00 & \underline{42.19} & \underline{52.17} & 52.02 & \underline{51.31} & 36.84 & 50.89 & 50.52 & 50.56 & 15.63 \\
\textbf{HoPE (Ours)} & \underline{54.89} & {\bf 56.36} & {\bf 55.70} & \textbf{45.12} & {\bf 52.31} & \textbf{52.97} & \textbf{51.66} & \textbf{46.27} & {\bf 51.79} & {\bf 51.87} & {\bf 51.69} & \textbf{26.03} \\
\midrule
\rowcolor{gray!15}
\multicolumn{13}{c}{\rule[-5pt]{0pt}{16pt}\textit{ \hspace{2em} \textbf{Qwen2-7B-Video}}} \\
\midrule
Vanilla RoPE & 59.75  & 61.13 & 61.03 & 34.38 & 51.17 & 50.31 & 51.29 & 39.47 & 56.70 & 57.96 & 57.99 & 26.13 \\
M-RoPE  & 59.70 & 61.68 & 62.46 & \underline{46.88} & 52.27 & \underline{53.29} & 53.49 & \underline{50.00} & 56.81 & 57.77 & 58.37 & 23.43 \\
VideoRoPE & \underline{60.40}& \underline{61.82}  & \underline{62.51} & 45.31 & \underline{52.89} & 53.13 & \underline{53.82} & 47.37 & \underline{57.51} & \underline{59.00} & \underline{59.13} & \underline{26.52} \\
\textbf{HoPE (Ours)} & {\bf 61.09} & {\bf 63.48} & {\bf 63.85} & \textbf{50.01} & {\bf 54.11} & {\bf 55.09} & {\bf 55.34} & \textbf{51.22} & \textbf{57.74} & \textbf{59.33} & \textbf{59.44} & \textbf{27.34} \\
\bottomrule[1pt]
\end{tabular}
}
\label{tab:main_results}
\end{table}

\subsection{Results on Long Video Understanding}

In this section, we provide a comprehensive comparison of HoPE and different RoPE variants in long video understanding. From Table~\ref{tab:main_results}, we observe that: \textbf{(1)} HoPE consistently outperforms all baselines across nearly all benchmarks, context lengths, and backbone sizes. Specifically, under the 7B model scale and 32k context lengths, HoPE surpasses vanilla RoPE by 2.82, 4.05, and 1.45 on MLVU, LongVideoBench, and Video-MME, respectively. This confirms its effectiveness and generalizability in multimodal long-context modeling. \textbf{(2)} The effectiveness of HoPE scales with backbone size. For instance, when the size of the backbone LLM increases from 2B to 7B, HoPE's performance gain on LongVideoBench (32k) significantly increases from 0.66 to 4.05 compared to vanilla RoPE. Notably, the performance gap between different methods on the 2B scale is less significant, probably due to the limited capabilities of the backbone LLM. \textbf{(3)} For context lengths under 64k, performance on Video-MME drops substantially, while the impact on MLVU and LongVideoBench is less pronounced. This suggests that extrapolating to extreme context lengths (e.g., up to 8x) remains highly challenging.

\begin{figure}[t]
    \centering
    \includegraphics[width=1\textwidth]{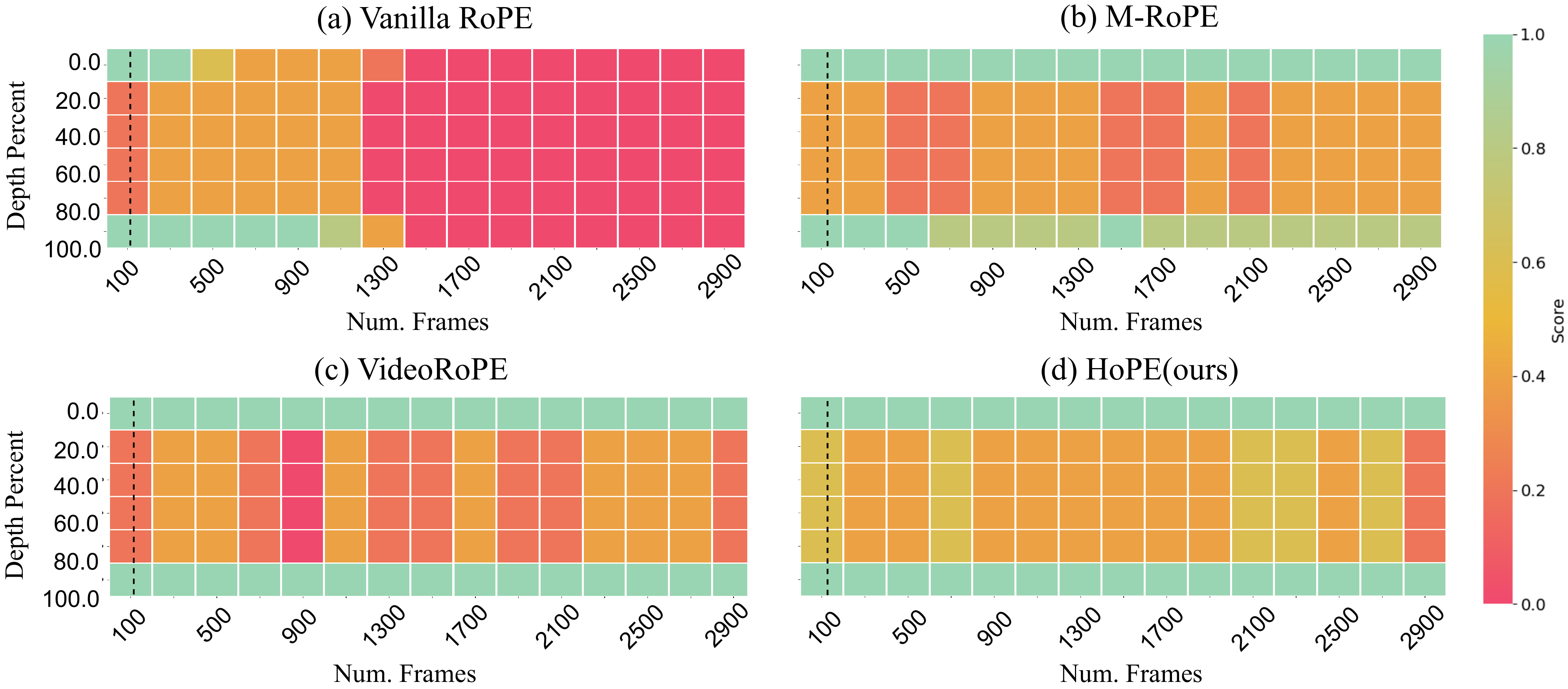}
    \caption{\textbf{Performance comparison on long video retrieval task (V-NIAH).} Here, each frame corresponds to 144 tokens. Cell colors indicate model accuracy (red: low, green: high), and the black dotted line marks the training context length (8k).}
    \label{fig:3}
\end{figure}

\subsection{Results on Long Video Retrieval}
\label{sec:exp_retrieval}

We evaluate HoPE against other RoPE variants on V-NIAH \cite{zhang2024long} to demonstrate the superiority of our method in long video retrieval, where VLMs are required to identify specific frames in a video to answer the question. Figure~\ref{fig:3} demonstrates that multimodal RoPEs significantly outperform vanilla RoPE, supporting our claim in Proposition~\ref{prop:1drope_for_video} that the flattening operation in vanilla RoPE hinders spatial-temporal modeling. Furthermore, HoPE achieves better extrapolation than M-RoPE and VideoRoPE, confirming its effectiveness in multimodal long-context modeling. Quantitative results in Table~\ref{tab:quant_VNIAH} show that HoPE surpasses the best baseline by a significant margin of 22.23\%.

\subsection{Analysis}

In this section, we first conduct ablation studies to analyze the effectiveness of each component in HoPE. We then present a comprehensive analysis exploring how different factors, including task type, context length, and the scaling factor of visual tokens, interact and impact model performance.

\textbf{Ablation Studies.} We conduct a series of ablation experiments to evaluate the impact of each component in HoPE and summarize the results in Table~\ref{tab:ablate}. According to the results, we observe that : \begin{wrapfigure}{r}{0.47\textwidth}
\centering
\begin{minipage}{0.47\textwidth}
\captionsetup{type=table}
\caption{\textbf{Ablation results on Video-MME from 8k to 64k.}
Here, HFA: hybrid frequency allocation, DTS: dynamic temporal scaling.}
\renewcommand{\arraystretch}{1.1}
\small
\setlength{\tabcolsep}{3pt}
\begin{tabular}{lcccc}
\toprule[1pt]
Method & 8k & 16k & 32k & 64k \\
\midrule
Vanilla RoPE & 56.70 & 57.96 & 57.99 & 26.13 \\
+ 3D structure & 56.81 & 57.77 & 58.37 & 23.43  \\
+ 3D + HFA & 57.66 & 59.19 & 59.31 &  26.98 \\
\rowcolor{myblue}
+ 3D + HFA + DTS & \textbf{57.74} & \textbf{59.33} & \textbf{59.44} & \textbf{27.34} \\
\bottomrule[1pt]
\end{tabular}
\label{tab:ablate}
\end{minipage}
\vspace{-1em}
\end{wrapfigure}(1) The 3D structure effectively improves the performance of vanilla RoPE in multimodal contexts, supporting our Proposition~\ref{prop:1drope_for_video}. (2) Based on the 3D structure, the hybrid frequency allocation (HFA) strategy further enhances long-range semantic modeling, achieving an average improvement of 1.69 across all context lengths. (3) The dynamic temporal scaling (DTS) mechanism facilitates VLMs' robustness to varying video speeds in real-world scenarios, yielding further performance gain. By combining the above strategies, our HoPE achieves the best overall performance across different context lengths in multimodal long-context modeling.

\begin{wraptable}{r}{0.55\textwidth}
\centering
\caption{\textbf{Ablation studies on test-time scaling factor selection.} We find that long video understanding generally benefits from larger scaling factors, while long video retrieval yields better results with smaller ones.}
\scriptsize
\begin{tabular}{cccccc}
\toprule
\multirow{2}{*}{\centering \textbf{Scaling Factor} \(\mathbf{\gamma}\)} & \multicolumn{4}{c}{\textbf{LongVideoBench}} & \multirow{2}{*}{\centering \textbf{V-NIAH}} \\
\cmidrule{2-5}
 & 8k & 16k & 32k & 64k \\
\midrule
0.50  & \textbf{54.48} & 54.29 & 54.36 & 52.63 & 60.89 \\
0.75  & 54.36 & 54.97 & 54.72 & 52.63 & \textbf{63.56}\\
1.00  & 54.11 & 54.48 & 54.97 & 52.63 & 62.67\\
1.25  & 54.11 & 54.84 & \textbf{55.70} & 52.63 & 62.67 \\
1.50  & 54.11 & \textbf{55.09} & 55.34 & 51.22 & 61.78 \\
\bottomrule
\end{tabular}
\label{tab:scale_ablate}
\end{wraptable}
\textbf{Impact of Test-time Scaling Factor Selection.} We conduct further experiments to investigate how different scaling factors \(\gamma\) in our dynamic temporal scaling mechanism impact the performance of video-related tasks. We summarize the results on V-NIAH and LongVideoBench in Table~\ref{tab:scale_ablate}. Our main observations are as follows: \textbf{(1)} \textit{Long video retrieval generally prefers smaller scaling factors.} As shown in Table~\ref{tab:scale_ablate}, when we utilize smaller scaling factors \(\gamma\) during inference, the performance on V-NIAH improves. We attribute this to the substantial length of 1-hour videos (3,000 frames), which far exceeds the training length (128 frames). In such cases, smaller scaling factors indirectly prevent the spatial position indices from becoming excessively large (see Equation~\ref{eq:dts}), thereby providing a better guarantee for the semantic preference property. Therefore, we set \(\gamma=0.75\) for long video retrieval. \textbf{(2)} \textit{Long video understanding generally benefits from larger scaling factors.} In contrast to retrieval, we find that long video understanding is relatively insensitive to the choice of scaling factor when the input context length is close to the training length. However, as the input length increases, employing larger scaling factors (\(\gamma>1\)) results in better performance. We hypothesize that while smaller scaling factors help preserve the semantic preference property, larger scaling factors are beneficial for maintaining spatial details (also see Equation~\ref{eq:dts}), which are crucial for complex understanding tasks. This introduces a natural tradeoff between semantic preference and spatial detail preservation. Compared to long video retrieval (3,000 frames, roughly 432k tokens), where extended temporal distances can significantly degrade semantic preference, in long video understanding tasks with context lengths of 16k–32k, the negative impact on semantic preference is relatively small. At the same time, the positive effect of larger scaling factors on capturing spatial details outweighs the semantic preference loss, making larger scaling factors overall more effective for complex video understanding. In our experiments, we set \(\gamma=1.5\) for long video understanding.

\section{Related Work}
\textbf{Position Embedding in LLMs.} Rotary Position Embedding (RoPE) \cite{su2024roformer} has become a common choice for position embedding in modern LLMs \cite{touvron2023llama, bai2023qwen, jiang2024mixtral, base_of_rope}. As discussed in Section~\ref{sec:preli}, RoPE achieves this success through rotating query and key vectors, encoding \textit{relative} position information through an \textit{absolute} positional encoding approach. Despite its success, several works have pointed out that No Position Embedding (NoPE) still works for decoder-only LLMs, arguing that the causal attention mechanism implicitly learns \textit{absolute} position information \cite{haviv2022transformer, kazemnejad2023impact, barbero2025round}. These works even suggest that NoPE outperforms RoPE in out-of-distribution (OOD) scenarios. However, this observation remains unexplored in multimodal settings, where positional encoding strategies may have different implications for cross-modal interactions. Based on Lemma~\ref{lemma:semantic_pref_condition}, we find that incorporating NoPE's zero frequency strategy indeed improves the length generalization of multimodal RoPE. 

\textbf{Multimodal Position Embedding in VLMs.} In VLMs \cite{liu2023visual, wang2024qwen2, chen2024internvl, ye2024mplug, lu2024deepseek}, images are first processed by vision encoders and then flattened into 1D tokens. Several early models \cite{liu2023visual, ye2024mplug, lu2024deepseek} rely on vanilla RoPE for positional encoding, which distorts spatial-temporal locality (see Section~\ref{sec:analysis}) and limits VLMs' long-context capability. Recently, Qwen2-VL \cite{wang2024qwen2} introduced M-RoPE, which extends 1D RoPE to multimodal settings by assigning distinct frequency ranges to different positional components. Specifically, M-RoPE allocates the \textit{highest} frequencies to the temporal component $t$, while distributing the lower frequencies sequentially to the spatial components $x$ and $y$. Conversely, VideoRoPE \cite{wei2025videorope} allocates the \textit{lowest} frequencies to $t$ to capture long-range dependencies, achieving stronger length generalization. However, these allocation strategies mainly rely on heuristics, lacking in-depth theoretical
analysis. In contrast, our work theoretically analyzes how different frequency allocation strategies impact the performance of multimodal RoPE. By zeroing out low frequencies for temporal modeling, our proposed HoPE provides the strongest theoretical guarantee for long-range semantic modeling. HoPE's strength is further enhanced by its dynamic temporal scaling of visual tokens, which enables robust temporal learning during training and flexible scaling during inference. By integrating these advantages, HoPE achieves state-of-the-art performance in long video understanding and retrieval tasks, making it well-suited for long context VLMs.

\section{Conclusion}
\label{sec:conclusion}
This paper theoretically analyzes the limitations of current multimodal RoPE variants. Our analysis reveals that: (1) vanilla RoPE inherently fails in spatial-temporal modeling; (2) keeping all frequencies in multimodal RoPE is unreliable in capturing long-range semantic similarities; (3) temporal scaling of lengthy visual tokens should include both compression and expansion to accommodate various video speeds. Consequently, we introduce HoPE, a hybrid of position embedding designed to enhance the long-context capabilities of VLMs. HoPE proposes a hybrid frequency allocation strategy to facilitate long-range semantic modeling, and a dynamic temporal scaling mechanism to enhance VLMs' robustness to varying video speeds in real-world scenarios. Experimental results on long video understanding and long video retrieval tasks demonstrate that HoPE consistently outperforms existing methods across diverse context lengths and backbone sizes, confirming its effectiveness.

\textbf{Limitations.} While HoPE’s performance gains scale from 2B to 7B backbones, our work does not use larger models or training data. We observe that the performance of all methods degrades significantly at 64k, though HoPE remains the most robust. While these resource-constrained evaluations are essential for uncovering genuine algorithmic benefits of multimodal RoPE, we note that training with more data, particularly long-context data, could further improve length generalization.

\bibliography{refs.bib}

\begin{thebibliography}{10}

\bibitem{liu2023visual}
Haotian Liu, Chunyuan Li, Qingyang Wu, and Yong~Jae Lee.
\newblock Visual instruction tuning.
\newblock {\em Advances in neural information processing systems}, 36:34892--34916, 2023.

\bibitem{wang2024qwen2}
Peng Wang, Shuai Bai, Sinan Tan, Shijie Wang, Zhihao Fan, Jinze Bai, Keqin Chen, Xuejing Liu, Jialin Wang, Wenbin Ge, et~al.
\newblock Qwen2-vl: Enhancing vision-language model's perception of the world at any resolution.
\newblock {\em arXiv preprint arXiv:2409.12191}, 2024.

\bibitem{chen2024internvl}
Zhe Chen, Jiannan Wu, Wenhai Wang, Weijie Su, Guo Chen, Sen Xing, Muyan Zhong, Qinglong Zhang, Xizhou Zhu, Lewei Lu, et~al.
\newblock Internvl: Scaling up vision foundation models and aligning for generic visual-linguistic tasks.
\newblock In {\em Proceedings of the IEEE/CVF conference on computer vision and pattern recognition}, pages 24185--24198, 2024.

\bibitem{ye2024mplug}
Qinghao Ye, Haiyang Xu, Jiabo Ye, Ming Yan, Anwen Hu, Haowei Liu, Qi~Qian, Ji~Zhang, and Fei Huang.
\newblock mplug-owl2: Revolutionizing multi-modal large language model with modality collaboration.
\newblock In {\em Proceedings of the ieee/cvf conference on computer vision and pattern recognition}, pages 13040--13051, 2024.

\bibitem{lu2024deepseek}
Haoyu Lu, Wen Liu, Bo~Zhang, Bingxuan Wang, Kai Dong, Bo~Liu, Jingxiang Sun, Tongzheng Ren, Zhuoshu Li, Hao Yang, et~al.
\newblock Deepseek-vl: towards real-world vision-language understanding.
\newblock {\em arXiv preprint arXiv:2403.05525}, 2024.

\bibitem{chen2024spatialvlm}
Boyuan Chen, Zhuo Xu, Sean Kirmani, Brain Ichter, Dorsa Sadigh, Leonidas Guibas, and Fei Xia.
\newblock Spatialvlm: Endowing vision-language models with spatial reasoning capabilities.
\newblock In {\em Proceedings of the IEEE/CVF Conference on Computer Vision and Pattern Recognition}, pages 14455--14465, 2024.

\bibitem{xia2025mmedrag}
Peng Xia, Kangyu Zhu, Haoran Li, Tianze Wang, Weijia Shi, Sheng Wang, Linjun Zhang, James Zou, and Huaxiu Yao.
\newblock {MM}ed-{RAG}: Versatile multimodal {RAG} system for medical vision language models.
\newblock In {\em The Thirteenth International Conference on Learning Representations}, 2025.

\bibitem{xia2024rule}
Peng Xia, Kangyu Zhu, Haoran Li, Hongtu Zhu, Yun Li, Gang Li, Linjun Zhang, and Huaxiu Yao.
\newblock Rule: Reliable multimodal rag for factuality in medical vision language models.
\newblock In {\em Proceedings of the 2024 Conference on Empirical Methods in Natural Language Processing}, pages 1081--1093, 2024.

\bibitem{xia2025mmedagent}
Peng Xia, Jinglu Wang, Yibo Peng, Kaide Zeng, Xian Wu, Xiangru Tang, Hongtu Zhu, Yun Li, Shujie Liu, Yan Lu, et~al.
\newblock Mmedagent-rl: Optimizing multi-agent collaboration for multimodal medical reasoning.
\newblock {\em arXiv preprint arXiv:2506.00555}, 2025.

\bibitem{yang2023exploring}
Xu~Yang, Yongliang Wu, Mingzhuo Yang, Haokun Chen, and Xin Geng.
\newblock Exploring diverse in-context configurations for image captioning.
\newblock {\em Advances in Neural Information Processing Systems}, 36:40924--40943, 2023.

\bibitem{fei2023transferable}
Junjie Fei, Teng Wang, Jinrui Zhang, Zhenyu He, Chengjie Wang, and Feng Zheng.
\newblock Transferable decoding with visual entities for zero-shot image captioning.
\newblock In {\em Proceedings of the IEEE/CVF international conference on computer vision}, pages 3136--3146, 2023.

\bibitem{li2024generative}
Yongqi Li, Wenjie Wang, Leigang Qu, Liqiang Nie, Wenjie Li, and Tat-Seng Chua.
\newblock Generative cross-modal retrieval: Memorizing images in multimodal language models for retrieval and beyond.
\newblock In {\em Proceedings of the 62nd Annual Meeting of the Association for Computational Linguistics (Volume 1: Long Papers)}, pages 11851--11861, 2024.

\bibitem{wang2024multimodal}
Yabing Wang, Le~Wang, Qiang Zhou, Zhibin Wang, Hao Li, Gang Hua, and Wei Tang.
\newblock Multimodal llm enhanced cross-lingual cross-modal retrieval.
\newblock In {\em Proceedings of the 32nd ACM International Conference on Multimedia}, pages 8296--8305, 2024.

\bibitem{wang2022learning}
Feng Wang, Manling Li, Xudong Lin, Hairong Lv, Alexander~G Schwing, and Heng Ji.
\newblock Learning to decompose visual features with latent textual prompts.
\newblock {\em ICLR}, 2022.

\bibitem{wang2024sclip}
Feng Wang, Jieru Mei, and Alan Yuille.
\newblock Sclip: Rethinking self-attention for dense vision-language inference.
\newblock In {\em European Conference on Computer Vision}, pages 315--332. Springer, 2024.

\bibitem{wang2025scaling}
Feng Wang, Yaodong Yu, Guoyizhe Wei, Wei Shao, Yuyin Zhou, Alan Yuille, and Cihang Xie.
\newblock Scaling laws in patchification: An image is worth 50,176 tokens and more.
\newblock {\em ICML}, 2025.

\bibitem{zhang2024long}
Peiyuan Zhang, Kaichen Zhang, Bo~Li, Guangtao Zeng, Jingkang Yang, Yuanhan Zhang, Ziyue Wang, Haoran Tan, Chunyuan Li, and Ziwei Liu.
\newblock Long context transfer from language to vision.
\newblock {\em arXiv preprint arXiv:2406.16852}, 2024.

\bibitem{chen2024sharegpt4video}
Lin Chen, Xilin Wei, Jinsong Li, Xiaoyi Dong, Pan Zhang, Yuhang Zang, Zehui Chen, Haodong Duan, Zhenyu Tang, Li~Yuan, et~al.
\newblock Sharegpt4video: Improving video understanding and generation with better captions.
\newblock {\em Advances in Neural Information Processing Systems}, 37:19472--19495, 2024.

\bibitem{cheng2024videollama}
Zesen Cheng, Sicong Leng, Hang Zhang, Yifei Xin, Xin Li, Guanzheng Chen, Yongxin Zhu, Wenqi Zhang, Ziyang Luo, Deli Zhao, et~al.
\newblock Videollama 2: Advancing spatial-temporal modeling and audio understanding in video-llms.
\newblock {\em arXiv preprint arXiv:2406.07476}, 2024.

\bibitem{maaz2024video}
Muhammad Maaz, Hanoona Rasheed, Salman Khan, and Fahad Khan.
\newblock Video-chatgpt: Towards detailed video understanding via large vision and language models.
\newblock In {\em Proceedings of the 62nd Annual Meeting of the Association for Computational Linguistics (Volume 1: Long Papers)}, pages 12585--12602, 2024.

\bibitem{wang2024needle}
Weiyun Wang, Shuibo Zhang, Yiming Ren, Yuchen Duan, Tiantong Li, Shuo Liu, Mengkang Hu, Zhe Chen, Kaipeng Zhang, Lewei Lu, et~al.
\newblock Needle in a multimodal haystack.
\newblock {\em Advances in Neural Information Processing Systems}, 37:20540--20565, 2024.

\bibitem{wang2025multimodal}
Hengyi Wang, Haizhou Shi, Shiwei Tan, Weiyi Qin, Wenyuan Wang, Tuny Zhang, Akshay Nambi, Tanuja Ganu, and Hao Wang.
\newblock Multimodal needle in a haystack: Benchmarking long-context capability of multimodal large language models.
\newblock In {\em Proceedings of the 2025 Conference of the North American Chapter of the Association for Computational Linguistics}, 2025.

\bibitem{su2024roformer}
Jianlin Su, Murtadha Ahmed, Yu~Lu, Shengfeng Pan, Wen Bo, and Yunfeng Liu.
\newblock Roformer: Enhanced transformer with rotary position embedding.
\newblock {\em Neurocomputing}, 568:127063, 2024.

\bibitem{touvron2023llama}
Hugo Touvron, Louis Martin, Kevin Stone, Peter Albert, Amjad Almahairi, Yasmine Babaei, Nikolay Bashlykov, Soumya Batra, Prajjwal Bhargava, Shruti Bhosale, et~al.
\newblock Llama 2: Open foundation and fine-tuned chat models.
\newblock {\em arXiv preprint arXiv:2307.09288}, 2023.

\bibitem{bai2023qwen}
Jinze Bai, Shuai Bai, Yunfei Chu, Zeyu Cui, Kai Dang, Xiaodong Deng, Yang Fan, Wenbin Ge, Yu~Han, Fei Huang, et~al.
\newblock Qwen technical report.
\newblock {\em arXiv preprint arXiv:2309.16609}, 2023.

\bibitem{jiang2024mixtral}
Albert~Q Jiang, Alexandre Sablayrolles, Antoine Roux, Arthur Mensch, Blanche Savary, Chris Bamford, Devendra~Singh Chaplot, Diego de~las Casas, Emma~Bou Hanna, Florian Bressand, et~al.
\newblock Mixtral of experts.
\newblock {\em arXiv preprint arXiv:2401.04088}, 2024.

\bibitem{heo2024rotary}
Byeongho Heo, Song Park, Dongyoon Han, and Sangdoo Yun.
\newblock Rotary position embedding for vision transformer.
\newblock In {\em European Conference on Computer Vision}, pages 289--305. Springer, 2024.

\bibitem{wei2025videorope}
Xilin Wei, Xiaoran Liu, Yuhang Zang, Xiaoyi Dong, Pan Zhang, Yuhang Cao, Jian Tong, Haodong Duan, Qipeng Guo, Jiaqi Wang, et~al.
\newblock Videorope: What makes for good video rotary position embedding?
\newblock In {\em International Conference on Machine Learning}, 2025.

\bibitem{vaswani2017attention}
Ashish Vaswani, Noam Shazeer, Niki Parmar, Jakob Uszkoreit, Llion Jones, Aidan~N Gomez, {\L}ukasz Kaiser, and Illia Polosukhin.
\newblock Attention is all you need.
\newblock {\em Advances in neural information processing systems}, 30, 2017.

\bibitem{haviv2022transformer}
Adi Haviv, Ori Ram, Ofir Press, Peter Izsak, and Omer Levy.
\newblock Transformer language models without positional encodings still learn positional information.
\newblock In {\em Findings of the Association for Computational Linguistics: EMNLP 2022}, pages 1382--1390, 2022.

\bibitem{kazemnejad2023impact}
Amirhossein Kazemnejad, Inkit Padhi, Karthikeyan Natesan~Ramamurthy, Payel Das, and Siva Reddy.
\newblock The impact of positional encoding on length generalization in transformers.
\newblock {\em Advances in Neural Information Processing Systems}, 36:24892--24928, 2023.

\bibitem{wang2024length}
Jie Wang, Tao Ji, Yuanbin Wu, Hang Yan, Tao Gui, Qi~Zhang, Xuan-Jing Huang, and Xiaoling Wang.
\newblock Length generalization of causal transformers without position encoding.
\newblock In {\em Findings of the Association for Computational Linguistics ACL 2024}, pages 14024--14040, 2024.

\bibitem{zhang2023video}
Hang Zhang, Xin Li, and Lidong Bing.
\newblock Video-llama: An instruction-tuned audio-visual language model for video understanding.
\newblock In {\em Proceedings of the 2023 Conference on Empirical Methods in Natural Language Processing: System Demonstrations}, pages 543--553, 2023.

\bibitem{lin2024video}
Bin Lin, Yang Ye, Bin Zhu, Jiaxi Cui, Munan Ning, Peng Jin, and Li~Yuan.
\newblock Video-llava: Learning united visual representation by alignment before projection.
\newblock In {\em Proceedings of the 2024 Conference on Empirical Methods in Natural Language Processing}, pages 5971--5984, 2024.

\bibitem{li2024lite}
Haoran Li, Junqi Liu, Zexian Wang, Shiyuan Luo, Xiaowei Jia, and Huaxiu Yao.
\newblock {LITE}: Modeling environmental ecosystems with multimodal large language models.
\newblock In {\em First Conference on Language Modeling}, 2024.

\bibitem{dosovitskiy2021an}
Alexey Dosovitskiy, Lucas Beyer, Alexander Kolesnikov, Dirk Weissenborn, Xiaohua Zhai, Thomas Unterthiner, Mostafa Dehghani, Matthias Minderer, Georg Heigold, Sylvain Gelly, Jakob Uszkoreit, and Neil Houlsby.
\newblock An image is worth 16x16 words: Transformers for image recognition at scale.
\newblock In {\em International Conference on Learning Representations}, 2021.

\bibitem{xiao2024efficient}
Guangxuan Xiao, Yuandong Tian, Beidi Chen, Song Han, and Mike Lewis.
\newblock Efficient streaming language models with attention sinks.
\newblock In {\em The Twelfth International Conference on Learning Representations}, 2024.

\bibitem{barbero2025round}
Federico Barbero, Alex Vitvitskyi, Christos Perivolaropoulos, Razvan Pascanu, and Petar Veli{\v{c}}kovi{\'c}.
\newblock Round and round we go! what makes rotary positional encodings useful?
\newblock In {\em The Thirteenth International Conference on Learning Representations}, 2025.

\bibitem{yang2024qwen2}
An~Yang, Baosong Yang, Beichen Zhang, Binyuan Hui, Bo~Zheng, Bowen Yu, Chengyuan Li, Dayiheng Liu, Fei Huang, Haoran Wei, et~al.
\newblock Qwen2. 5 technical report.
\newblock {\em arXiv preprint arXiv:2412.15115}, 2024.

\bibitem{zhang2024video}
Yuanhan Zhang, Jinming Wu, Wei Li, Bo~Li, Zejun Ma, Ziwei Liu, and Chunyuan Li.
\newblock Video instruction tuning with synthetic data.
\newblock {\em arXiv preprint arXiv:2410.02713}, 2024.

\bibitem{wu2024longvideobench}
Haoning Wu, Dongxu Li, Bei Chen, and Junnan Li.
\newblock Longvideobench: A benchmark for long-context interleaved video-language understanding.
\newblock {\em Advances in Neural Information Processing Systems}, 37:28828--28857, 2024.

\bibitem{fu2024video}
Chaoyou Fu, Yuhan Dai, Yongdong Luo, Lei Li, Shuhuai Ren, Renrui Zhang, Zihan Wang, Chenyu Zhou, Yunhang Shen, Mengdan Zhang, et~al.
\newblock Video-mme: The first-ever comprehensive evaluation benchmark of multi-modal llms in video analysis.
\newblock {\em arXiv preprint arXiv:2405.21075}, 2024.

\bibitem{zhou2024mlvu}
Junjie Zhou, Yan Shu, Bo~Zhao, Boya Wu, Shitao Xiao, Xi~Yang, Yongping Xiong, Bo~Zhang, Tiejun Huang, and Zheng Liu.
\newblock Mlvu: A comprehensive benchmark for multi-task long video understanding.
\newblock {\em arXiv preprint arXiv:2406.04264}, 2024.

\bibitem{base_of_rope}
Mingyu Xu, Xin Men, Bingning Wang, Qingyu Zhang, Hongyu Lin, Yaojie Lu, Xianpei Han, and Weipeng Chen.
\newblock Base of rope bounds context length.
\newblock In A.~Globerson, L.~Mackey, D.~Belgrave, A.~Fan, U.~Paquet, J.~Tomczak, and C.~Zhang, editors, {\em Advances in Neural Information Processing Systems}, volume~37, pages 87386--87410. Curran Associates, Inc., 2024.

\bibitem{ren2022shunted}
Sucheng Ren, Daquan Zhou, Shengfeng He, Jiashi Feng, and Xinchao Wang.
\newblock Shunted self-attention via multi-scale token aggregation.
\newblock In {\em Proceedings of the IEEE/CVF conference on computer vision and pattern recognition}, pages 10853--10862, 2022.

\bibitem{ren2025flowar}
Sucheng Ren, Qihang Yu, Ju~He, Xiaohui Shen, Alan Yuille, and Liang-Chieh Chen.
\newblock Flowar: Scale-wise autoregressive image generation meets flow matching.
\newblock In {\em ICML}, 2025.

\bibitem{ren2025xar}
Sucheng Ren, Qihang Yu, Ju~He, Xiaohui Shen, Alan Yuille, and Liang-Chieh Chen.
\newblock Beyond next-token: Next-x prediction for autoregressive visual generation.
\newblock In {\em ICCV}, 2025.

\end{thebibliography}
\bibliographystyle{unsrt}







\newpage 
\appendix

\section{Proofs}
\label{app:proofs}
In this section, we provide detailed proofs for the theoretical statements presented in this paper.

\subsection{Vanilla RoPE Fails in Spatial-Temporal Structure}
\label{app:proof_1drope}

\textbf{Proposition~\ref{prop:1drope_for_video}.} 
\textit{Given any query \(\mathbf{q}\) at position \((t,x,y)\) and a relative distance of 1 in spatial or temporal dimensions, the flattening operation in 1D RoPE distorts the relative distance with a magnitude dependent on the frame resolution.}

\begin{proof}

Consider a video of shape $T \times H \times W$, where each token at position $(t,x,y)$ is flattened by
\[
f(t,x,y) = tHW + xW + y.
\]
Now consider two types of local neighbors:

1. \textbf{Spatial neighbors within the same frame}:

   Let \((t,x,y)\) and \((t,x+1,y)\) be adjacent in the spatial dimension. Then,
    \begin{equation}
        |f(t,x+1,y) - f(t,x,y)| = |((x+1)W + y) - (xW + y)| = W.
    \end{equation}
   Note that a relative distance of 1 in \(x\) becomes \(W\) after flattening, which grows linearly with the frame width.

2. \textbf{Temporal neighbors at the same spatial position}:
   Let \((t,x,y)\) and \((t+1,x,y)\) be adjacent in time. Then,
   \begin{equation}
       |f(t+1,x,y) - f(t,x,y)| = |(t+1)HW + xW + y - (tHW + xW + y)| = HW.
   \end{equation}
   For a 1-frame shift in time, the index difference becomes \(HW\), which grows with spatial resolution.

In both cases, spatially or temporally adjacent tokens are mapped to indices with significant differences. Since vanilla RoPE incorporates positional information based on these 1D index differences, such flattening leads to distorted spatial-temporal relationships.
\end{proof}

\subsection{Semantic Preference Property}
\label{app:semantic_pref}

We now prove that the frequency allocation strategies in current multimodal RoPEs are unreliable in capturing semantic similarities over extended contexts, as defined in Definition~\ref{def:semantic_pref}.

\textbf{Definition~\ref{def:semantic_pref}.} \textbf{(}Semantic Preference\textbf{).} For any query vector $\mathbf{q}$ and a semantically similar key vector $\mathbf{k}'$ that can be expressed as $\mathbf{k}' = \mathbf{q} + \delta$ where $\delta$ is a zero-mean perturbation, the attention score with RoPE should satisfy:
\begin{equation}
\mathbb{E}_{\mathbf{q},\mathbf{k},\delta}[\mathbf{q}\mathbf{R}_{\Delta t \Delta x \Delta y }\mathbf{k}'-\mathbf{q}\mathbf{R}_{\Delta t \Delta x \Delta y }\mathbf{k}]\ge 0,
\end{equation}
where $\mathbf{k}$ is the key vector of a semantically unrelated token. This preference should hold regardless of the relative distance (\(\Delta t, \Delta x, \Delta y\)) between query-key pairs.

Firstly, we use Lemma~\ref{lem:neg_prob_monotone} to show why using lower frequencies for temporal modeling is more ideal in multimodal RoPE. Intuitively, larger rotation angles (frequencies) are more likely to produce negative cosine similarity values between semantically related tokens under long-context scenarios.

\begin{lemma}\label{lem:neg_prob_monotone}
Let \(\Delta t\) be drawn uniformly from \(\{0,1,\dots,L-1\}\), and define
\[
  P_{\neg}(\theta)
  \;=\;
  \frac{1}{L}\,\bigl|\{\Delta : \cos(\theta\,\Delta t)<0\}\bigr|.
\]
Then for any \(L>1\):
\begin{enumerate}
  \item If \(0<\theta<\frac{\pi}{2(L-1)}\), then \(P_{\neg}(\theta)=0\).
  \item For \(\theta \ge \frac{\pi}{2(L-1)}\), \(P_{\neg}(\theta)\) is non-decreasing in \(\theta\).
  \item \(\displaystyle\lim_{\theta\to\infty}P_{\neg}(\theta) = \tfrac12.\)
\end{enumerate}
\end{lemma}

\begin{proof}[Proof of Lemma~\ref{lem:neg_prob_monotone}]
1. \textbf{No negative region for small \(\theta\).}
   If \(0<\theta<\frac{\pi}{2(L-1)}\), then for every \(\Delta t \in\{0,\dots,L-1\}\) we have
   \[
     0 \;\le\; \theta\,\Delta t\;<\;\theta\,(L-1) \;<\;\pi/2,
   \]
   so \(\cos(\theta\,\Delta t)>0\).  Hence \(P_{\neg}(\theta)=0\).

2. \textbf{Monotonicity once the first zero enters.}
   As soon as \(\theta\ge \frac{\pi}{2(L-1)}\), the point satisfying \(\theta\,\Delta t=\pi/2\) lies in \(\{0,\dots,L-1\}\).  Each further increase in \(\theta\) extends the interval of length \(\theta(L-1)\), adding more half-periods of cosine.  Each added half-period contains exactly one “negative” region of length \(\pi\).  Therefore the count
   \(\bigl|\{\Delta : \cos(\theta\Delta)<0\}\bigr|\) (and hence \(P_{\neg}(\theta)\)) can only stay the same or increase, up to \(O(1/L)\) rounding errors on the discrete grid.

3. \textbf{Limit to one half for large \(\theta\).}
   For large \(\theta\), the values \(\{\theta\,\Delta\}_{\Delta=0}^{L-1}\) become equidistributed mod \(2\pi\). Since the negative region \(\{\;x\;\bmod\;2\pi\mid \cos x<0\}\) has total length \(\pi\) over each \(2\pi\)-cycle, one finds
   \[
     \lim_{\theta\to\infty}P_{\neg}(\theta)
     \;=\;
     \frac{\pi}{2\pi}
     \;=\;
     \tfrac12.
   \]
\end{proof}

We can now prove the frequency allocation strategies in current multimodal RoPE cannot reliably maintain the semantic preference property, i.e., semantically similar tokens should receive higher attention than semantically unrelated pairs.

\textbf{Theorem~\ref{theorem:semantic_preference}.} \textit{Let $X = [x_1, x_2, \dots, x_L]$ be an input sequence, and let RoPE use any fixed set of temporal frequencies (e.g., highest or lowest). Then there exists a critical length $L_c$ such that for all $L \ge L_c$, the semantic preference property (Definition~\ref{def:semantic_pref}) is violated.}

\begin{proof}
We first recall the definition of multimodal RoPE, where the rotation matrix is partitioned to encode different dimensions:
\[\mathbf{R}_{t,x,y} = \operatorname{diag}(\mathbf{R}_t, \mathbf{R}_x,\mathbf{R}_y),\]
where $\mathbf{R}_t$, $\mathbf{R}_x$, and $\mathbf{R}_y$ are rotation matrices applied to temporal, horizontal spatial, and vertical spatial dimensions, respectively, with each dimension carrying a frequency of \(\theta_i=b^{-2i/d}, i \in \{0,\dots,d/2-1\}\). Note that the (\(t,x,y\)) ordering is purely notational and does not constrain the actual dimension allocation strategy.

Assume that each component of the query vector \(\mathbf{q}\) is independently and identically distributed with mean \(\mu\) and variance \(\sigma^2\). We denote key vector that is semantically similar to \(\mathbf{q}\) as \(\mathbf{k}'=\mathbf{q}+\delta \), where \(\delta\) is a zero-mean perturbation. The semantically unrelated key vector \(\mathbf{k}\) is independently drawn with the same distribution as $\mathbf{q}$. Let \(\Delta t, \Delta x, \Delta y\) denote relative temporal and spatial distances between the query and each key. According to Definition~\ref{def:semantic_pref}, the semantic preference property requires that:
\begin{equation}
    \centering
\begin{aligned}
    &\mathbb{E}_{\mathbf{q},\mathbf{k},\delta}[\mathbf{q}\mathbf{R}_{\Delta t, \Delta x, \Delta y}\mathbf{k}'^{\top}-\mathbf{q}\mathbf{R}_{\Delta t, \Delta x, \Delta y}\mathbf{k}^{\top}]\\ 
    &=\mathbb{E}_{\mathbf{q},\mathbf{k},\delta}[\mathbf{q}\mathbf{R}_{\Delta t, \Delta x, \Delta y}\mathbf{(q+\delta)}^{\top}-\mathbf{q}\mathbf{R}_{\Delta t, \Delta x, \Delta y}\mathbf{k}^{\top}] 
    \\
    &=\mathbb{E}_{\mathbf{q}}[\mathbf{q}\mathbf{R}_{\Delta t, \Delta x, \Delta y}\mathbf{q}^{\top}] -\mathbb{E}_{\mathbf{q},\mathbf{k}}[\mathbf{q}\mathbf{R}_{\Delta t, \Delta x, \Delta y}\mathbf{k}^{\top}] 
    \\
    &=\mathbb{E}_{\mathbf{q}}[\mathbf{q}\mathbf{R}_{\Delta t, \Delta x, \Delta y}\mathbf{q}^{\top}] -\mu^2\mathbf{R}_{\Delta t, \Delta x, \Delta y}
    \\
    &=\sum_{i \in i_t}2(\mu^2+\sigma^2)\mathrm{cos}(\Delta t)\theta_i+\sum_{i \in i_x}2(\mu^2+\sigma^2)\mathrm{cos}(\Delta x)\theta_i+\sum_{i \in i_y}2(\mu^2+\sigma^2)\mathrm{cos}(\Delta y)\theta_i-\\
    &\sum_{i \in i_t}2\mu^2\mathrm{cos}(\Delta t)\theta_i+\sum_{i \in i_x}2\mu^2\mathrm{cos}(\Delta x)\theta_i+\sum_{i \in i_y}2\mu^2\mathrm{cos}(\Delta y)\theta_i \\
    &=\sum_{i \in i_t}2\sigma^2\mathrm{cos}(\Delta t \cdot \theta_i)+\sum_{i \in i_x}2\sigma^2\mathrm{cos}(\Delta x \cdot \theta_i)+\sum_{i \in i_y}2\sigma^2\mathrm{cos}(\Delta y \cdot \theta_i) \ge 0,
\end{aligned}
\end{equation}
where \(i_t, i_x,i_y\) denote dimensions allocated to encode temporal \((t)\), horizontal spatial \((x)\), and vertical spatial \((y)\) information. To satisfy the semantic preference property (Definition~\ref{def:semantic_pref}), the expected attention between a query and its semantically similar key should remain higher than that for an unrelated key, regardless of their relative distance. This implies the following condition must hold universally:
\begin{equation}
\begin{aligned}
    \sum_{i \in i_t}2\sigma^2\mathrm{cos}(\Delta t \cdot \theta_i)+\sum_{i \in i_x}2\sigma^2\mathrm{cos}(\Delta x \cdot \theta_i)+\sum_{i \in i_y}2\sigma^2\mathrm{cos}(\Delta y \cdot \theta_i) \ge 0, \\
    \Delta t\in\{0,1,\dots,L-1\}, \Delta x\in\{0,1,\dots,H\}, \Delta y\in\{0,1,\dots,W\}.
\end{aligned}
\label{eq:semantic_pref}
\end{equation}
Now consider a long-context scenario, where \(L \gg H, W\), we can now theoretically prove that why VideoRoPE \cite{wei2025videorope} (using lowest frequencies for \(t\)) is better than M-RoPE
(using highest frequencies for \(t\)) in maintaining semantic preference property in long contexts. Simply, by Lemma~\ref{lem:neg_prob_monotone}, we show that when the context length \(L\) is sufficiently large, the probability that \(cos(\Delta t \cdot \theta_i)\) leads to negative values becomes higher when \(\theta_i\) becomes larger. Therefore, lower frequencies, which rotate less, are less likely to violate the semantic preference property. 

However, despite using the lowest frequencies, VideoRoPE still fails to guarantee that the semantic preference property holds for all context lengths (Equation~\ref{eq:semantic_pref}). Let VideoRoPE allocate only the smallest frequency to the temporal dimensions, instead of \(|i_t|\) smallest frequencies:
\[
  \theta_{\min}
  \;=\;
  b^{-2(\frac d2-1)/d},
\]
so that in Equation~\eqref{eq:semantic_pref} the temporal sum reduces to:
\[
  2\sigma^2\ |i_t| \cos(\Delta t \cdot \theta_{\min}).
\]
Here, under the reasonable assumption that semantically related tokens co-occur in nearby spatial positions across frames, the spatial sums in Equation~\ref{eq:semantic_pref} remains non-negative. Thus we only consider the temporal sum in Equation~\ref{eq:semantic_pref}:
\[
  2\sigma^2 |i_t| \cos(\Delta t \cdot \theta_{\min}).
\]
Now pick any context length \(L\) so large that there exists
\[
  \Delta t
  \;\in\;
  \{0,1,\dots,L-1\}
  \quad\text{with}\quad
  \Delta t\;\theta_{\min}
  \;\in\;
  \bigl(\tfrac{\pi}{2},\,\tfrac{3\pi}{2}\bigr).
\]
Such a \(\Delta t\) indeed exists as soon as
\(\theta_{\min}\,(L-1)>\tfrac{\pi}{2}\), i.e. for any
\[
  L \;>\; L_c=\frac{\pi}{2\,\theta_{\min}} +1.
\]
For that choice of \(\Delta t\), we have
\[
  \cos(\Delta t\ \cdot \theta_{\min})<0,
\]
and hence the left‐hand side of
Equation~\eqref{eq:semantic_pref} becomes
\[
  2\sigma^2\ |i_t| \cos(\Delta t\ \cdot \theta_{\min})
  <0.
\]
This single counterexample \((\Delta t,\Delta x,\Delta y)\) violates the
semantic preference condition, since no further temporal frequencies are
available to “rescue” the sum. Therefore, despite using the lowest frequencies for temporal modeling, VideoRoPE still fails to guarantee the semantic preference property. In conclusion, all frequency allocation strategies in current multimodal RoPEs fail to maintain the semantic preference property in Definition~\ref{def:semantic_pref}, completing the proof.
\end{proof}

\section{Further Experimental Details}
\label{app:exp_detail}
In this section, we provide further details of our experiments, including benchmark descriptions, experimental settings, and further results.

\begin{figure}[t]
    \centering
    \includegraphics[width=0.98\textwidth]{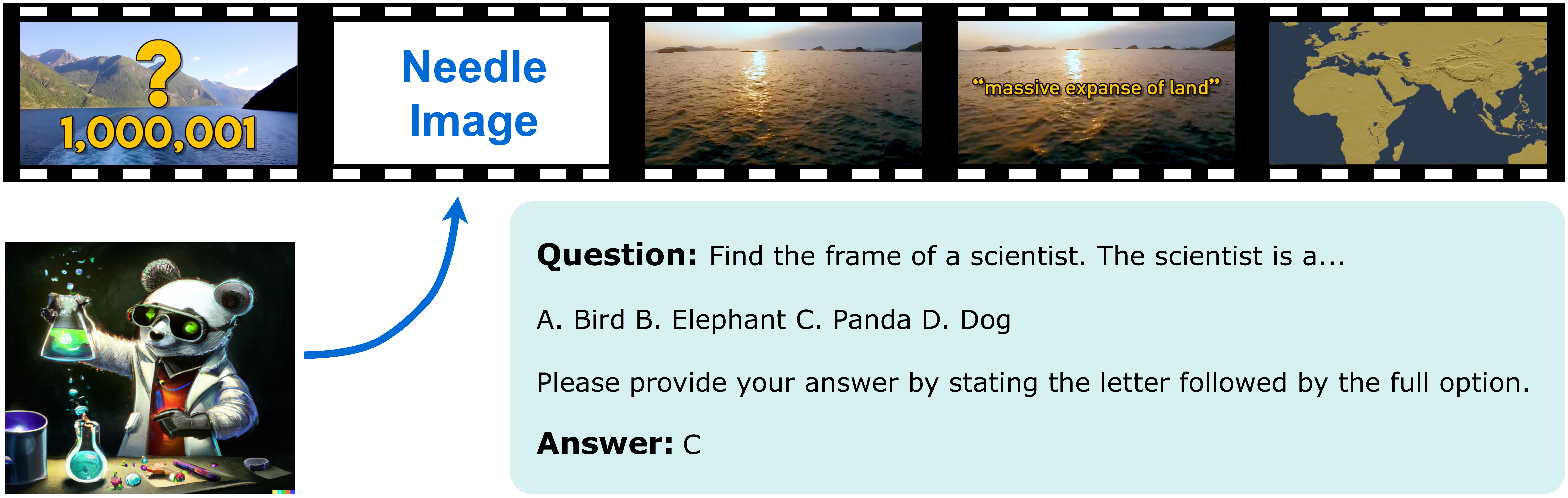}
    \caption{Illustration of V-NIAH, which consists of a randomly inserted needle image, a haystack video, and a specific question related to the needle.}
    \label{fig:V-NIAH}
\end{figure}

\subsection{Detailed Benchmark Description}
\label{app:benchmark}

There is a growing interest in video generation and understanding~\cite{ren2022shunted, ren2025flowar, ren2025xar, lin2024video, cheng2024videollama}, given their broad applications in content creation and analysis. In this subsection, we provide detailed descriptions of the video benchmarks we used in the experiments, i.e.,  LongVideoBench \cite{wu2024longvideobench}, Video-MME \cite{fu2024video}, MLVU \cite{zhou2024mlvu}, and V-NIAH \cite{zhang2024long}.
\begin{itemize}[leftmargin=8pt]
    \item  \textbf{LongVideoBench} is a comprehensive benchmark for evaluating Vision-Language Models on long video understanding tasks. Unlike traditional video benchmarks that focus on short clips under one minute, this dataset features videos ranging from 8 seconds to 1 hour across diverse sources, including everyday life, movies, knowledge, and news. The benchmark encompasses 17 fine-grained question categories organized into two levels: perception and relation. In our experiment, questions that are free from subtitles are retained.
    
    \item  \textbf{Video-MME} is a full-spectrum evaluation benchmark of Vision-Language Models in video analysis, spanning 6 primary visual domains with
    30 subfields to ensure generalizability. It features temporal diversity by incorporating both short- (<2 minutes), medium- (4-15 minutes), and long-term videos (30-60 minutes), ranging from 11 seconds to 1 hour.
    
    \item  \textbf{MLVU} is a high-quality benchmark designed to evaluate the video understanding capabilities of Vision-Language Models. The temporal duration of videos within MLVU spans from 3 minutes to 2 hours, covering genres such as movies, life records, and egocentric videos. In our experiment, we evaluate all methods on the following multiple-choice tasks: Action Count, Action Order, Topic Reasoning, Ego Reasoning, Needle QA, Plot QA, and Anomaly Recognition.

    \item \textbf{V-NIAH} is a challenging benchmark designed to evaluate VLMs' ability to identify specific frames within long videos. In this task, a "needle" image is inserted into a "haystack" video, and the VLMs are required to answer specific questions about this "needle" image, as shown in Figure~\ref{fig:V-NIAH}. Following the settings in V-NIAH \cite{zhang2024long}, we utilize a haystack video with 1-hour duration (3,000 frames). The needle image is inserted at 20\% depth intervals (e.g., a frame depth of 0\% would place the needle image at the very beginning of the video.)
    
\end{itemize}

\begin{table}[h]
\centering
\caption{Quantitative performance of different RoPE variants on V-NIAH. Here, we report the average accuracy across different context lengths and frame depths.}
\vspace{0.5em}
\begin{tabular}{l|cccc}
\toprule
       & Vanilla RoPE & M-RoPE & VideoRoPE & HoPE (ours) \\ \midrule
V-NIAH & 21.00         & 47.11  & 
\underline{52.00}      & \textbf{63.56}       \\ 
\bottomrule
\end{tabular}
\label{tab:quant_VNIAH}
\end{table}

\subsection{Quantitative Results on V-NIAH}

Here, we provide the quantitative results of different RoPE variants on long video retrieval task in Table~\ref{tab:quant_VNIAH}. It can be observed that our HoPE demonstrates a 22.23\% improvement compared to the best baseline, justifying its effectiveness in multimodal long-context modeling.

\subsection{Ideal Condition for Semantic Preference}
\label{app:ideal_sem_pref}

As discussed after Theorem~\ref{theorem:hybrid_frequency}, the semantic preference property (Definition~\ref{def:semantic_pref}) invariably holds for any context length \(t\) and spatial size \(x,y\) when we set \(|i_t|=d/4,|i_x|=|i_y|=d/8\) and \(\theta_i=0, i \in i_t\), since Lemma~\ref{lemma:semantic_pref_condition} reduces to:

\[
\sum_{i=0}^{d/8-1}2\sigma^2(2+\mathrm{cos}(\Delta x \cdot \theta_{2i})+\mathrm{cos}(\Delta y \cdot \theta_{2i+1})) \geq 0.
\]

In our original HoPE implementation, the frequencies allocated to \(t,x,y\) are \(16,24,24\), respectively, with the lowest \(16\) frequencies for  \(t\) set to zero. For this proposed variant (HoPE-X), we redistribute these allocations to \(32,16,16\) for \(t,x,y\), respectively, while setting the lowest \(32\) frequencies for \(t\) to zero. To evaluate the comparative effectiveness of these configurations, we conduct further experiments on LongVideoBench. 

\begin{table}[h]
\centering
\caption{Performance comparison between HoPE-X and HoPE.}
\vspace{1em}
\begin{tabular}{l|cccc}
\toprule
\multirow{2}{*}{\centering Method} & \multicolumn{4}{c}{LongVideoBench} \\
\cmidrule{2-5}
 & 8k & 16k & 32k & 64k \\
\midrule
HoPE-X & 52.68 & 52.73 & 53.01 & 46.32 \\
HoPE   & 54.11 & 55.09 & 55.34 & 51.22 \\
\bottomrule
\end{tabular}
\label{tab:ideal}
\end{table}

Table~\ref{tab:ideal} demonstrates that HoPE consistently outperforms HoPE-X across diverse context lengths. We deduce that the inferior performance of HoPE-X is due to its decreased dimensions allocated for spatial modeling. While this configuration helps to maintain the semantic preference property, it negatively impacts HoPE-X's ability to model local features. Therefore, it is necessary to keep adequate dimensions for spatial modeling in multimodal RoPE.



\newpage
\section*{NeurIPS Paper Checklist}

\begin{enumerate}

\item {\bf Claims}
    \item[] Question: Do the main claims made in the abstract and introduction accurately reflect the paper's contributions and scope?
    \item[] Answer: \answerYes{} 
    \item[] Justification: Our main claims are summarized in Figure~\ref{fig:1}. Section~\ref{sec:analysis} provides theoretical analysis and Section~\ref{sec:method} offers detailed methodology.

\item {\bf Limitations}
    \item[] Question: Does the paper discuss the limitations of the work performed by the authors?
    \item[] Answer: \answerYes{} 
    \item[] Justification: Yes, please see Section~\ref{sec:conclusion}.

\item {\bf Theory assumptions and proofs}
    \item[] Question: For each theoretical result, does the paper provide the full set of assumptions and a complete (and correct) proof?
    \item[] Answer: \answerYes{} 
    \item[] Justification: We provide the proofs in Appendix~\ref{app:proofs}.

    \item {\bf Experimental result reproducibility}
    \item[] Question: Does the paper fully disclose all the information needed to reproduce the main experimental results of the paper to the extent that it affects the main claims and/or conclusions of the paper (regardless of whether the code and data are provided or not)?
    \item[] Answer: \answerYes{} 
    \item[] Justification: Experimental setups are provided in Section~\ref{sec:exp_setup}. 

\item {\bf Open access to data and code}
    \item[] Question: Does the paper provide open access to the data and code, with sufficient instructions to faithfully reproduce the main experimental results, as described in supplemental material?
    \item[] Answer: \answerYes{} 
    \item[] Justification: We use fully public datasets in our experiments and the details are provided in Section~\ref{sec:exp_setup} and Appendix~\ref{app:benchmark}. Code will be released in camera-ready version.

\item {\bf Experimental setting/details}
    \item[] Question: Does the paper specify all the training and test details (e.g., data splits, hyperparameters, how they were chosen, type of optimizer, etc.) necessary to understand the results?
    \item[] Answer: \answerYes{} 
    \item[] Justification: Please refer to Section~\ref{sec:exp_setup}.

\item {\bf Experiment statistical significance}
    \item[] Question: Does the paper report error bars suitably and correctly defined or other appropriate information about the statistical significance of the experiments?
    \item[] Answer: \answerNo{} 
    \item[] Justification: In our experiments, the variance between different runs is negligible. Additionally, our training and evaluation pipelines are deterministic.

\item {\bf Experiments compute resources}
    \item[] Question: For each experiment, does the paper provide sufficient information on the computer resources (type of compute workers, memory, time of execution) needed to reproduce the experiments?
    \item[] Answer: \answerYes{} 
    \item[] Justification: Please refer to Section~\ref{sec:exp_setup}.
    
\item {\bf Code of ethics}
    \item[] Question: Does the research conducted in the paper conform, in every respect, with the NeurIPS Code of Ethics \url{https://neurips.cc/public/EthicsGuidelines}?
    \item[] Answer: \answerYes{} 
    \item[] Justification: We followed the NeurIPS Code of Ethics.

\item {\bf Broader impacts}
    \item[] Question: Does the paper discuss both potential positive societal impacts and negative societal impacts of the work performed?
    \item[] Answer: \answerNA{} 
    \item[] Justification: This work focuses on improving the long-context capabilities of Vision-Language Models, which uses public datasets and is purely academic. We believe it has no direct societal impact.
    
\item {\bf Safeguards}
    \item[] Question: Does the paper describe safeguards that have been put in place for responsible release of data or models that have a high risk for misuse (e.g., pretrained language models, image generators, or scraped datasets)?
    \item[] Answer: \answerNA{} 
    \item[] Justification: This paper does not pose safety risks.

\item {\bf Licenses for existing assets}
    \item[] Question: Are the creators or original owners of assets (e.g., code, data, models), used in the paper, properly credited and are the license and terms of use explicitly mentioned and properly respected?
    \item[] Answer: \answerYes{} 
    \item[] Justification: We credited them in appropriate ways and followed their licenses.

\item {\bf New assets}
    \item[] Question: Are new assets introduced in the paper well documented and is the documentation provided alongside the assets?
    \item[] Answer: \answerNA{} 
    \item[] Justification: No new assets are introduced.

\item {\bf Crowdsourcing and research with human subjects}
    \item[] Question: For crowdsourcing experiments and research with human subjects, does the paper include the full text of instructions given to participants and screenshots, if applicable, as well as details about compensation (if any)? 
    \item[] Answer: \answerNA{} 
    \item[] Justification: This paper does not include crowdsourcing experiments and research with human subjects.

\item {\bf Institutional review board (IRB) approvals or equivalent for research with human subjects}
    \item[] Question: Does the paper describe potential risks incurred by study participants, whether such risks were disclosed to the subjects, and whether Institutional Review Board (IRB) approvals (or an equivalent approval/review based on the requirements of your country or institution) were obtained?
    \item[] Answer: \answerNA{} 
    \item[] Justification: The paper does not involve crowdsourcing nor research with human subjects.

\item {\bf Declaration of LLM usage}
    \item[] Question: Does the paper describe the usage of LLMs if it is an important, original, or non-standard component of the core methods in this research? Note that if the LLM is used only for writing, editing, or formatting purposes and does not impact the core methodology, scientific rigorousness, or originality of the research, declaration is not required.
    \item[] Answer: \answerNA{} 
    \item[] Justification: The core method development in this research does not involve LLMs as any important, original, or non-standard components.

\end{enumerate}

\end{document}